\documentclass{article}

\usepackage{arxiv}
\usepackage{xcolor}

\definecolor{darkred}{rgb}{0.6, 0, 0}
\usepackage[utf8]{inputenc} % allow utf-8 input
\usepackage[T1]{fontenc}    % use 8-bit T1 fonts
\usepackage{hyperref}       % hyperlinks
\usepackage{url}            % simple URL typesetting
\usepackage{booktabs}       % professional-quality tables
\usepackage{amsfonts}       % blackboard math symbols
\usepackage{nicefrac}       % compact symbols for 1/2, etc.
\usepackage{microtype}      % microtypography
\usepackage{lipsum}		% Can be removed after putting your text content
\usepackage{graphicx}
\usepackage{natbib}
\usepackage{doi}
\usepackage{algorithm}
\usepackage{algorithmic}
\usepackage{amsmath} 
\usepackage{amssymb}
\usepackage{comment}
\usepackage{mathbbol}

\usepackage{setspace}       % for double spacing

\usepackage{newfloat}
\usepackage{listings}
\usepackage{subcaption}

\newtheorem{assumption}{Assumption}
\newtheorem{theorem}{Theorem}[section]

\title{Online Survival Analysis:\\ A Bandit Approach under Cox PH Model}

\date{} 					% Or removing it

\author{ 
Yang Xu%\thanks{} 
\\
	Department of Statistics\\
	North Carolina State University\\
	\texttt{yxu63@ncsu.edu} \\
	\And
        Wenbin Lu \\
	Department of Statistics\\
	North Carolina State University\\
	\texttt{wlu4@ncsu.edu} \\
 	\And
        Rui Song \\
	Department of Statistics\\
	North Carolina State University\\
	\texttt{songray@gmail.com} \\
}

\renewcommand{\headeright}{}
\renewcommand{\undertitle}{}
\renewcommand{\shorttitle}{Online Survival Analysis: A Bandit Approach under Cox PH Model}

\hypersetup{
pdftitle={Online Survival Analysis: A Bandit Approach under Cox PH Model},
pdfsubject={q-bio.NC, q-bio.QM},
pdfauthor={David S.~Hippocampus, Elias D.~Striatum},
pdfkeywords={First keyword, Second keyword, More},
}

\begin{document}
\maketitle

\begin{abstract}
Survival analysis is a widely used statistical framework for modeling time-to-event data under censoring. Classical methods, such as the Cox proportional hazards (Cox PH) model, offer a semiparametric approach to estimating the effects of covariates on the hazard function. Despite its importance, survival analysis has been largely unexplored in online settings, particularly within the bandit framework, where decisions must be made sequentially to optimize treatments as new data arrive over time. In this work, we take an initial step toward integrating survival analysis into a purely online learning setting under the Cox PH model, addressing key challenges including staggered entry, delayed feedback, and right censoring. We adapt three canonical bandit algorithms to balance exploration and exploitation, with theoretical guarantees of sublinear regret bounds. Extensive simulations and semi-real experiments using SEER cancer data demonstrate that our approach enables rapid and effective learning of near-optimal treatment policies.
\end{abstract}

% keywords can be removed
\keywords{Survival Analysis \and Contextual Bandits \and Cox PH \and Staggered Entry \and Right Censoring \and Delayed Reward}

\section{Introduction}

Survival analysis is an essential tool for modeling time-to-event data in various fields, including precision medicine to prolong patient survival in cancer studies \citep{altman1995review, nagy2021pancancer, gyorffy2010online}, predictive maintenance to reduce costs in engineering systems \citep{widodo2011machine, papathanasiou2023machine}, and customer retention analysis to predict the duration of platform engagement \citep{lu2002predicting, harrison2002customer}. Classical survival models can be categorized into three types: parametric models such as the Accelerated Failure Time (AFT) model \citep{kalbfleisch2002statistical}, semi-parametric models such as the Cox proportional hazards (PH) model \citep{cox1972regression}, and non-parametric models such as the Kaplan-Meier estimator \citep{kaplan1958nonparametric}. Among these, the Cox PH model remains widely adopted in practice due to its balance of flexibility and interpretability.

A common practical challenge arises when subjects enter the study at different times, known as \emph{staggered entry}. In this case, the risk set at each event time depends on the delayed entry, which requires careful accounting to correctly construct the likelihood function. Most existing survival studies, whether or not they handle staggered entry, are conducted in offline settings: researchers wait for lengthy survival data to be fully collected before analyzing it, making predictions, and optimizing decisions post hoc. While effective, this approach can be slow, costly, and inefficient, particularly in scenarios with long time-to-event outcomes.

For example, consider a pharmaceutical company testing a preventive drug for seasonal allergy symptoms in a Phase II or III trial. Classical practice involves assigning patients to treatment or control groups, waiting for a sufficient follow-up period, analyzing the resulting time-to-event data (in this case, the time until the next seasonal allergy episode), and then evaluating the drug’s efficacy and potential personalized treatment recommendations. This process can take months or even years, especially when events occur seasonally, delaying iterative learning and decision-making. An online updating system, such as bandits, could significantly shorten this feedback loop by adapting treatment recommendations as data accumulate.

Bandits, a special case of reinforcement learning, have been successfully applied in diverse domains such as recommendation systems \citep{silva2022multi}, robotics \citep{kober2013reinforcement}, and portfolio selection \citep{shen2015portfolio}. Bandit algorithms dynamically balance exploration and exploitation, continuously updating the system’s understanding to provide better action selection over time. Since individuals often enter studies at different times, survival data naturally fit into the bandit framework. An effective decision-making system should adapt to the newly observed outcomes of all past data, providing personalized recommendations to maximize survival or retention. To our knowledge, however, no existing work addresses online survival updates for time-to-event data in a bandit setting.

Several challenges complicate this integration. First, survival data inherently involve delayed and sometimes censored feedback, where the “reward” itself corresponds to the observed time-to-event. While recent bandit studies have considered delayed rewards \citep{tang2024stochastic, schlisselberg2025delay, zhang2025contextual}, they typically do not account for censoring, which further complicates modeling and prolongs the feedback loop. Second, although streaming survival models have been proposed \citep{xue2020online, wu2021online, su2023divide}, these methods primarily focus on accurate prediction rather than decision-making and are not designed to handle staggered entry. Third, designing effective exploration strategies and establishing regret guarantees is particularly challenging with survival data, as the dynamic risk set introduces additional variability that breaks the standard martingale concentration tools typically used to derive uniform bounds for parameter estimates.

In this paper, we make the following contributions:

\textbf{An integrated online bandit framework for survival data with right censoring.} Our method simultaneously addresses right censoring, staggered entry, and delayed feedback, enabling timely and adaptive decision-making rather than requiring offline analysis after all censoring outcomes are observed. This capability is particularly valuable for applications such as clinical trials, precision medicine, and customer churn modeling.

\textbf{Flexible algorithms for online survival updates.} We extensively study how the risk set evolves in a bandit setting, extending classical offline survival analysis structure to an online setting without imposing stochastic assumptions on the covariates. We adapt standard bandit strategies, including Epsilon-Greedy (EG), Upper Confidence Bound (UCB), and Thompson Sampling (TS), to enable flexible exploration in the survival context.

\textbf{Confidence Ellipsoid and Sublinear regret guarantees.} The nature of staggered entry in bandits invalidates classical martingale concentration arguments, creating substantial theoretical challenges. We establish the first confidence ellipsoid for online survival settings and provide a regret bound of order $\tilde{O}(d\sqrt{T})$ flexible enough to applying to EG, UCB, and TS. Semi-real experiments on SEER cancer data and simulations further demonstrate that our approach rapidly converges to near-optimal policies.

\section{Related Work}
Our work addresses survival analysis with staggered entry in an online setting, and is mainly related to three main threads of literature. Below, we summarize the state-of-the-art methods in each area.

\paragraph{Offline Survival Analysis with Staggered Entry} 
Offline survival analysis with staggered entry has been extensively studied over the past few decades, primarily in the context of the Cox PH model and its variants. Notably, \cite{andersen1982cox} formulated the Cox PH model using counting processes, rigorously handling staggered entry and establishing asymptotic properties. Later, \cite{sellke1983sequential} further consolidated the theory by studying martingale properties and approximations under the calendar time scale. Comprehensive reviews \citep{kalbfleisch2002statistical, therneau2000cox} summarize popular survival models and practical extensions. More recently, there have been adaptations of survival models to modern nonparametric approaches, including neural networks without staggered entry \citep{katzman2018deepsurv} and Bayesian nonparametric models with staggered entry \citep{zhang2023weibull}. Overall, while classical models such as the Cox PH model remain widely used in practice, the field is gradually moving toward more flexible modeling approaches.

\paragraph{Offline Survival Analysis with Streaming Data} 
Another relevant thread concerns survival analysis with streaming data, where survival data arrive sequentially in blocks. Representative works include \cite{shaker2014survival}, which uses a sliding-window approach assuming the hazard rate is constant within each window; \cite{xue2020online}, which applies a first-order Taylor expansion for online approximation of the partial likelihood; and \cite{wang2021fast}, which proposes a divide-and-conquer approach with one-step linear and least-squares approximations of the partial likelihood. However, these methods do not operate in a truly online setting: they assume that, once a batch of data arrives, the reward and censoring status are immediately observable. In contrast, in the online bandit setting we consider, samples arrive sequentially and rewards may be delayed, introducing additional complexity. Therefore, while related, these methods are not directly comparable to our framework.
%We note a recent concurrent work \cite{anderer2026online} that also studies bandit frameworks for survival data. Compared to this work, our approach considers a more general $K$-arm setting without restrictive assumptions on subject arrival times, accommodates a broader class of exploration strategies beyond Thompson sampling, and establishes cleaner sublinear regret guarantees that hold under more flexibly defined event rates.

\paragraph{Online Bandits with Delayed Reward} 
A third thread comes from the bandit literature, specifically multi-armed and contextual bandits with delayed rewards. Notable works include \cite{tang2024stochastic, schlisselberg2025delay}, which extend classical UCB and Successive Elimination algorithms to handle reward delays, and \cite{zhang2025contextual}, which extends these ideas to linear contextual bandits. However, these approaches focus exclusively on the no-censoring case, which differs from classical survival analysis where rewards are not only delayed but also subject to right censoring. Thus, while these works are related conceptually, they are not directly comparable to our setting.

\section{Offline Cox PH Model with Staggered Entry}
Before introducing our online bandit framework, we first review how the Cox PH model accommodates staggered entry in a purely offline setting. 

The Cox PH model \citep{cox1972regression} is among the most widely used tools in survival analysis. It models the hazard of an event at survival time $s$ as
\begin{equation*}
    \lambda(s\mid X) = \lambda_0(s)\exp(X^\top\beta), \quad s \geq 0,
\end{equation*}
where $\lambda_0(s)$ is the baseline hazard function, $X$ is a feature vector, and $\beta$ is an unknown coefficient vector quantifying covariate effects.

In the idealized case where all subjects enter the study simultaneously, a subject is considered \textit{at risk} at time $s$ if $s \leq Y_i \wedge C_i$, where $Y_i$ is the true survival time and $C_i$ is the censoring time. With staggered entry, however, it is necessary to consider two time scales that jointly determine the risk set: the \emph{calendar time} $\tau$, representing an absolute timestamp, and the \emph{survival time} $s$, representing the elapsed time since a subject’s entry \citep{sellke1983sequential, bilias1997towards}. 

For $N$ subjects with staggered entry times $\tau_i$, the risk status of subject $i$ at calendar time $\tau$ depends on whether the subject has entered the study. This can be captured by the entry-time offset $(\tau - \tau_i)^+ = \max\{\tau - \tau_i, 0\}$, which can be viewed as an additional censoring mechanism alongside $C_i$. Consequently, the risk set at $(\tau, s)$ is defined as
\begin{equation}
    \mathcal{R}(\tau,s) = \{i\in[N]: s\leq Y_i\wedge C_i\wedge (\tau-\tau_i)^+\},
\end{equation}
which consists of subjects who are under observation at calendar time $\tau$ and have survived for at least $s$ time units since entry. We then define $N_i(\tau,s)=\boldsymbol{1}\{Y_i\leq s\wedge C_i\wedge (\tau-\tau_i)^+\}$, which represents the event counting process for subject $i$ evaluated at the calendar-survival time pair $(\tau,s)$. 

With these definitions, the log partial likelihood for $\beta$, denoted by $l(\tau,\beta)$, takes the form
\begin{equation}\label{eq:l_tau_beta}
\sum_{i=1}^N \int_0^\tau \Big\{X_i^\top\beta-\log \sum_{j\in\mathcal{R}(\tau,s)}\exp(X_j^\top\beta)\Big\}N_i(\tau,ds).
\end{equation}
Differentiating with respect to $\beta$ yields the score function
\begin{equation}\label{eq:score_func}
    U(\tau,\beta) := \sum_{i=1}^N \int_0^{\tau} \{X_i - \bar{X}(\tau,s)\} N_i(\tau,ds),
\end{equation}
where $\bar{X}(\tau,s) := \frac{\sum_{j\in \mathcal{R}(\tau,s)} X_j(s) \exp\{X_j(s)^\top \beta\}}{\sum_{j\in \mathcal{R}(\tau,s)}\exp\{X_j(s)^\top \beta\}}$. Solving $U(\tau,\beta)=0$ gives the maximum partial likelihood estimator $\hat{\beta}$, which can then be combined with an estimator of the baseline hazard function to derive survival function estimates and guide downstream decision making.

\section{Online Survival Bandits}
We now extend the discussion to the online setting and introduce our framework for adapting survival analysis to the bandit system. 

Suppose there are $T$ rounds in total. In each round $t\in[T]:=\{1,\dots,T\}$, a new subject $i \in[N]$ enters the study at calendar time $\tau_i$ with covariate information $S_i \in \mathbb{R}^{d_0}$ (e.g., age, gender, or prescription history). For clarity, we use $i$ to index subjects and $t$ to index rounds. Since one subject arrives per round, $N$ and $T$ can often be used interchangeably. Once subject $i$ enters at round $t$, they remain in the study for all subsequent rounds $\{t,\dots,T\}$.

The bandit system selects an action $a_i \in \mathcal{A}$ from a total of $|\mathcal{A}|=K$ choices (e.g., treatment assignment for patients, such as different drug doses), and determines a follow-up duration $C_i$, which we interpret as a censoring time in the survival analysis context. The right-censored outcome $R_i:= Y_i\wedge C_i :=\min\{Y_i,C_i\}$, where $Y_i$ denotes the patient’s true event time, is available only after calendar time $\tau_i + R_i$. Let $\delta_i = \mathbf{1}\{Y_i \le C_i\}$ indicate the event status, so that $\delta_i = 1$ if the event is observed (e.g., patient experienced the event during follow-up) and $\delta_i = 0$ otherwise. Consequently, the full data for subject $i$ is given by $\mathcal{D}_i := \{\tau_i,S_i, A_i, \delta_i, R_i\}$. 

However, note that $(\delta_i, R_i)$ is not observed immediately at round $t$ when subject $i$ enters, but only after a delay. Specifically, these outcomes become available in round $t'$ if $\tau_{t'} \geq \tau_i + R_i$. To formalize this, we define a status indicator
$$\eta_{i,t} = \boldsymbol{1}\{(\delta_i,R_i) \text{ is observed by calendar time }\tau_t \}.$$
Accordingly, we define the observed data for subject $i$ at round $t$ as $\mathcal{D}_{i,t} = \{\tau_i,S_i, A_i, \eta_{i,t}\delta_i, \eta_{i,t}R_i\}$. By construction, subject $i$ enters the study at round $i$, so $\mathcal{D}_{i,t}$ is defined only for $t \ge i$. When the observation window is sufficiently long for all outcomes to be observed, we recover the full data $\mathcal{D}_i = \mathcal{D}_{i,\infty}$. %During the follow-up window $[\tau_i, \tau_i + C_i)$, the censoring indicator and reward $\{\delta_i, R_i\}$ remain unobserved. 

Let $X= \Phi(S,A)\in\mathbb{R}^{d}$ denote the collection of processed features constructed from the covariates $S$ and action $A$, which serve as the feature input to the Cox PH model. Importantly, our theoretical analysis imposes no stochastic assumptions on the covariates $S$ or the derived features $X$. This design choice provides maximal flexibility for feature engineering and decision-making in practice.

For a discrete action space with a reasonably small number of actions $K$, a natural choice for $\Phi$ is $\Phi(S,A) = (S^\top\boldsymbol{1}\{A=1\},\dots, S^\top\boldsymbol{1}\{A=K\})^\top$, so that $d =  d_0K$. Accordingly, $\beta$ can be viewed as a $d_0K$-dimensional vector, decomposable as $\beta = (\beta_1^\top, \dots, \beta_K^\top)^\top$, where $\beta_k$ corresponds to the coefficients of covariates $S$ for subjects assigned to action $k$. Alternative formulations of $\Phi$, such as incorporating polynomial expansions or interaction terms, can also be adopted to capture more flexible relationships between $S$ and $A$. 

The primary objective in most survival studies is to infer the survival function. In the bandit setting, this translates into leveraging historical data to better estimate $S(\tau \mid X(a))$, where the dependence on the action $a$ is made explicit through $X(a)$, and to identify the optimal action $a$ that maximizes $S(\tau \mid X(a))$ for each subject. More generally, the goal can be extended to optimizing a functional $f(S(\tau \mid X(a)))$ under mild Lipschitz continuity assumptions, which will be detailed in the theory section. For example, $f(S(\tau_0 \mid X(a)))$ can denote the mean restricted survival time $\text{MRST}(\tau_0 \mid X(a)) = \int_0^{\tau_0} S(\tau \mid X(a))\, d\tau$, which quantifies the expected survival time truncated at $\tau_0$. For concreteness, we focus on the survival probability at a prespecified time point $\tau_0$. This corresponds to the special case $f(x)=x$, with $S(\tau_0 \mid X(a))$ serving as the primary quantity of interest.

%a functional $f(S(\tau\mid X))$. Classical choices of $f$ include the survival function $f(S(\tau\mid X)) = S(\tau_0\mid X) = \mathbb{P}(Y > \tau_0\mid X)$ and the mean restricted survival time $\text{MRST}(\tau_0\mid X) = \int_0^{\tau_0} S(\tau\mid X) d\tau$, which quantify, respectively, the probability of surviving beyond time $\tau_0$ and the expected survival time truncated at $\tau_0$. Both metrics capture key aspects of the underlying time-to-event distribution. For illustration purpose, we focus on survival function $S(\tau_0)$ at a given time point $\tau_0$ as the primary quantity of interest. However, this formulation is general and can accommodate any $f(\cdot)$ under mild regularity conditions, which are detailed in the theoretical section.
\subsection{Optimization with Pure Exploitation}

Before introducing exploration in an online bandit setting, we first consider the case of \emph{pure exploitation}.

We begin by decomposing the information update from round $t$ to $t+1$, focusing on both the risk set $\mathcal{R}(\tau,s)$ and the partial likelihood function $l(\tau,\beta)$. 

Consider the risk set $\mathcal{R}(\tau,s)$, which depends on both calendar time $\tau$ and survival time $s$. For a fixed $\tau$, $\mathcal{R}(\tau,s)$ is a non-increasing set as $s \to \infty$, since larger $s$ makes it harder for subjects to satisfy the at-risk condition. Conversely, for a fixed $s$, $\mathcal{R}(\tau,s)$ is non-decreasing as $\tau \to \infty$. Larger calendar time increases $(\tau-\tau_i)^+$ for each subject $i$, thereby allowing subjects previously excluded by staggered entry to enter the risk set.

Based on the at-risk status change from round $t$ to $t+1$, subjects can be divided into three categories:
\begin{enumerate}
    \item $\eta_{i,t}\eta_{i,t+1}=1$: $\mathcal{D}_i$ is fully observed at round $t$;
    \item $\eta_{i,t}=0$ and $\eta_{i,t+1}=1$: $(\delta_i, R_i)$ is not available at round $t$ but becomes available at round $t+1$;
    \item $\eta_{i,t}=\eta_{i,t+1}=0$: $\mathcal{D}_i$ remains unobserved at round $t+1$.
\end{enumerate}
These three categories exhaust all possibilities, since $\eta_{i,t}=1$ and $\eta_{i,t+1}=0$ cannot occur by definition.

For group 1, the risk set does not change from $\mathcal{R}(\tau_t,s)$ to $\mathcal{R}(\tau_{t+1},s)$. For group 2, since $(\tau_t-\tau_i)^+ < R_i \le (\tau_{t+1}-\tau_i)^+$, we need to include these subjects in the updated risk set starting from round $t+1$. For group 3, although $(\delta_i,R_i)$ remain unobserved, we still have information update from them since $R_i > (\tau_{t+1}-\tau_i)^+$, which effectively enlarges the risk set for some $s$. After some calculations, the risk set update from round $t$ to $t+1$ can be written as
\begin{equation*}
\begin{aligned}
    \mathcal{R}(\tau_{t+1},s) =& \mathcal{R}(\tau_{t},s)\bigcup \big\{i\in[N]:\eta_{i,t}=0,(\tau_t-\tau_i)^+<s\leq (\tau_{t+1}-\tau_i)^+\wedge R_i\big\}.
\end{aligned}
\end{equation*}

Next, we consider the update of partial likelihood function from round $t$ to $t+1$. Recall that at calendar time $\tau$,  $l(\tau,\beta)$ is defined as
\[
\sum_{i=1}^N \int_0^\tau \Big\{ X_i^\top \beta - \log \sum_{j\in\mathcal{R}(\tau,s)} \exp(X_j^\top \beta) \Big\} N_i(\tau, ds).
\]
Since both $N(\tau,s)$ and $\mathcal{R}(\tau,s)$ change from $\tau_t$ to $\tau_{t+1}$, $l(\tau,\beta)$ must adapt accordingly to reflect the change. While $l(\tau_{t+1},\beta)$ cannot be expressed simply as $l(\tau_t,\beta)$ plus the contribution solely from subject $t+1$, an efficient online update is possible. By monitoring all subjects in groups 2 and 3 (i.e., for $\{i\in[N]:\eta_{i,t}=0\}$) whose information evolves over time, we can write
\[
l(\tau_{t+1},\beta) = l(\tau_t,\beta) + P_1 + P_2,
\]
where
$$
\begin{aligned}
P_1 & = \sum_{i=1}^N \int_0^{\tau_{t+1}} \Big\{ X_i^\top \beta - \log \sum_{j\in\mathcal{R}(\tau_{t+1},s)} \exp(X_j^\top \beta) \Big\}  \big( N_i(\tau_{t+1},ds) - N_i(\tau_t,ds) \big),
\end{aligned}
$$
and
$$
P_2 = \sum_{i=1}^N \int_0^{\tau_t} \log \frac{\sum_{j\in\mathcal{R}(\tau_{t},s)} \exp(X_j^\top \beta)}{\sum_{j\in\mathcal{R}(\tau_{t+1},s)} \exp(X_j^\top \beta)} \, N_i(\tau_t, ds).
$$
%This approach speeds up computation for large $T$, avoiding a full recalculation of the risk set and log-likelihood at each round.

\textbf{Remark. } \textit{(Computational Complexity)} \\
Let $N$ denote the total number of samples (rounds), $m$ the total number of observed events, and $p$ the dimension of the covariates. A naive approach refits the Cox proportional hazards model from scratch at each round $t$, using all $t$ accumulated samples. Since each Cox fit via Newton-Raphson requires $O(t p^2)$ time (dominated by Hessian computation), the total computational cost over $N$ rounds is $\sum_{t=1}^N O(t p^2) = O(N^2 p^2)$, which becomes prohibitively expensive as $N$ grows. 

In contrast, our method incrementally updates the partial likelihood by leveraging the decomposition $l(\tau_{t+1}, \beta) = l(\tau_t, \beta) + P_1 + P_2$, where $P_1$ accounts for contributions from newly observed events between $\tau_t$ and $\tau_{t+1}$, and $P_2$ corrects for changes in the risk sets due to staggered entry. Specifically, $P_1$ requires $O(\Delta m_t \cdot p)$ operations, where $\Delta m_t$ is the number of new events at round $t$, while $P_2$ requires $O(m_t)$ operations, where $m_t$ is the cumulative number of events up to round $t$. Summing over all rounds yields a total complexity of $O(m p + N m)$, which depends on the number of events rather than the square of the sample size (see Figure~\ref{fig:sim_runtime} in the Appendix for empirical runtime as the sample size increases). This approach avoids repeated Hessian computations and full risk set recalculations, providing computational savings when events are sparse ($m \ll N$) and enabling efficient online Cox updates for large-scale sequential decision-making.

Based on the updated partial likelihood, solving $U(\tau_t,\beta) = 0$ yields the coefficient estimator at round $t$ under pure exploitation, denoted by $\hat{\beta}_t^*$.
% Although there is no  There are two parts that need to be changed: First, for subjects in group 2), if they experienced the event ($\delta_i=1$), we need to add their partial likelihood to $l(\tau_{t+1},\beta)$; Second, for all subjects, we need to also update the risk set from $\mathcal{R}(\tau_t,s)$ to $\mathcal{R}(\tau_{t+1},s)$ to incorporate potentially new information. In short, 
% \begin{equation*}
%     l(\tau_{t+1},\beta) = l(\tau_{t},\beta) + P_1 + P_2,
% \end{equation*}
% where $P_1 = \sum_{i=1}^T \int_0^\tau \big\{X_i^\top\beta-\log \sum_{j\in\mathcal{R}(\tau_{t+1},s)}\exp(X_j^\top\beta)\big\} \big\{N_i(\tau_{t+1},ds)-N_i(\tau_{t},ds) \big\}$, and $P_2 = \sum_{i=1}^T \int_0^\tau \log \big\{\sum_{j\in\mathcal{R}(\tau_{t+1},s)}\exp(X_j^\top\beta)/\sum_{j\in\mathcal{R}(\tau_t,s)}\exp(X_j^\top\beta)\big\} N_i(\tau_{t},ds)$.
Since the goal is to maximize the survival function $S(\tau_0\mid X_t) = S_0(\tau_0)^{\exp(X_t^\top \beta)}$ using the current estimate $\hat{\beta}_t^*$, this reduces to selecting the action 
\[
\hat{a}^*_t := \arg\max_{a} S_0(\tau_0)^{\exp(X_t(a)^\top \hat{\beta}_t^*)}.
\] 
Since the baseline survival function $S_0(\tau_0) \in (0,1]$ is independent of $a$, the optimization above is equivalent to 
\[
\hat{a}^*_t = \arg\min_a X(a)^\top \hat{\beta}_t^*.
\] 
%providing a simple and efficient decision rule under pure exploitation.

\subsection{Online Learning with Exploration}
In online bandit learning, the central challenge lies in balancing exploration and exploitation. Having discussed the online case of pure exploitation, we now introduce exploration strategies and extend three classical bandit algorithms: Epsilon-Greedy (EG), Upper Confidence Bound (UCB), and Thompson Sampling (TS), to the setting of online survival analysis.

\textbf{EG.}  
Under EG exploration, the action at round $t$ is selected as
\begin{equation}\label{eq:EG_act}
a_{t}= (1-Z_{t})\cdot \arg\min_{a} {X}_{t}(a)^\top\hat{{\beta}}_{t-1}^* + Z_{t}\cdot \text{DU}(1,K),
\end{equation}
where $Z_{t}\sim \text{Ber}(\epsilon_{t})$, and $\text{DU}(1,K)$ denotes the discrete uniform distribution s.t. $\mathbb{P}(A_t=a)=\frac{1}{K}$ for any $a\in\mathcal{A}$.

\textbf{UCB.} 
Define ${\Sigma}_{t}(\beta):= I_t(\beta)^{-1}\in\mathbb{R}^{dK\times dK}$, where $I_t$ is the observed Fisher information process, given by
\begin{equation}\label{eq:I_beta}
\begin{aligned}
    I_t({\beta}) =\sum_{i=1}^N\int_0^{\tau_t} &\Big\{\sum_{j\in\mathcal{R}(\tau_t,s)}\omega_j(\tau_t,s) X_jX_j^\top-\bar{X}(\tau_t,s)\bar{X}(\tau_t,s)^\top\Big\} N_i(\tau_t,ds),
\end{aligned}
\end{equation}
where $\omega_j(\tau,s) = \frac{ \exp\{X_j(s)^\top \beta\}}{\sum_{i\in \mathcal{R}(\tau,s)} \exp\{X_i(s)^\top \beta\}}$.

Since we minimize $X(a)^\top\beta$, the upper confidence bound under $A_{t}=a\in\mathcal{A}$ can be derived as
\begin{equation}
\begin{aligned}
    \text{UCB}_{t,a} 
    \leftarrow -{X}_{t}(a)^\top\hat{{\beta}}_{t-1}^* + \alpha \cdot \sqrt{{X}_{t}(a)^\top{\Sigma}_{t-1}(\hat{{\beta}}_{t-1}^*){X}_{t}(a)},
\end{aligned}
\end{equation}      
where $\alpha$ is a hyperparameter that controls the exploration-exploitation tradeoff. UCB-based exploration selects the arm $a_{t}$ by
\begin{equation}\label{eq:UCB_act}
    a_{t}=\arg\max_{a} \text{UCB}_{t,a}.
\end{equation}

\textbf{TS.}
For a prior $\pi(\beta) =\mathcal{N}(\mu_0,\Sigma_0)$, the log-posterior of $\beta$ given data $\mathcal{D}$ is
\begin{equation*}
    \begin{aligned}
    \log f(\tau,\beta\mid\mathcal{D}) &:= l(\tau,\beta\mid\mathcal{D})+\log \pi(\beta)\\
    &= \sum_{i=1}^N \int_0^\tau \Big\{X_i^\top\beta-\log \sum_{j\in\mathcal{R}(\tau,s)}\exp(X_j^\top\beta)\Big\}N_i(\tau,ds) -\frac{1}{2}(\beta-\mu_0)^\top\Sigma_0^{-1}(\beta-\mu_0) + C,
    \end{aligned}
\end{equation*}

where $C$ is a constant independent of $\beta$. Unlike in classical bandits, the posterior here has no closed form and does not follow a standard distribution. We therefore approximate it using either MCMC (e.g., Metropolis-Hastings, Hamiltonian Monte Carlo) or Laplace approximation. In our experiments, both Hamiltonian Monte Carlo (HMC) and Laplace methods yield stable performance. Once a posterior sample $\widetilde{\beta}_{t-1}$ is drawn, we choose the action
\begin{equation}\label{eq:TS_act}
    a_{t} = \arg\min_{a} {X}_{t}(a)^\top\widetilde{\beta}_{t-1}.
\end{equation}

\section{Theory} 
In bandit learning, the objective is to select personalized actions that maximize a functional of the time-to-event distribution, such as the survival probability $S(\tau_0\mid X(a))$ or the mean restricted survival time $\mathrm{MRST}(\tau_0\mid X(a))$. The cumulative regret over $T$ rounds is defined as
\begin{equation*}
    \text{Regret}(T) = \sum_{t=1}^T \mathbb{E}[f(S(\tau_0\mid X_t(a^*_t)))-f(S(\tau_0\mid X_t(a_t)))],
\end{equation*}
where $a^*_t$ is the optimal action, $a_t$ is the action chosen by the algorithm (e.g., EG, UCB, or TS), and the expectation is taken with respect to the randomness in $X(a)$. If we set $f(x)=x$, i.e., to optimize the survival probability $S(\tau_0\mid X)$, the cumulative regret becomes
\begin{equation*}
    \begin{aligned}
    \sum_{t=1}^T \mathbb{E}\Big[{S}_0(\tau_0)^{\exp\{X_t(a^*_t)^\top{\beta}\}} - {S}_0(\tau_0)^{\exp\{X_t(a_t)^\top{\beta}\}}\Big].
    \end{aligned}
\end{equation*}
Regret arises only when the algorithm selects a suboptimal arm ($a_t\neq a^*_t$). Define $g(z) = S_0(\tau_0)^{\exp(z)}$, which is strictly decreasing and smooth in $z$ since $S_0(\tau_0)\in(0,1]$. By the Mean Value Theorem, the order of $\text{Regret}(T)$ in $T$ is the same as that of $\sum_{t=1}^T \mathbb{E}[\Delta_t]$, where $\Delta_t := X_t(a_t)^\top\beta - X_t(a^*_t)^\top\beta\geq 0$. That is,
\begin{equation*}
\text{Regret}(T) \leq \sup_z|g'(z)| \cdot \sum_{t=1}^T \mathbb{E}[\Delta_t]\leq \frac{1}{e}\sum_{t=1}^T\mathbb{E}[\Delta_t].
\end{equation*}
For a general regret defined via a function $f(S(\tau_0 \mid X(a)))$, as long as $f$ is (uniformly) Lipschitz on $[0,1]$, there exists a constant $C_0$ such that $\text{Regret}(T)\leq C_0\sum_{t=1}^T\mathbb{E}[\Delta_t]$. Hence, the main challenge reduces to establishing an upper bound for $\sum_{t=1}^T \mathbb{E}[\Delta_t]$.

For each $t$, $\mathbb{E}[\Delta_t]$ has the same form as in linear contextual bandits. However, in the survival setting, $a_t$ is chosen according to \eqref{eq:EG_act}, \eqref{eq:UCB_act}, or \eqref{eq:TS_act}, where $\hat{\beta}_t^*$ is estimated from the Cox partial likelihood $l(\tau_t,\beta)$ under staggered entry. This creates a significant challenge for establishing regret upper bounds, since the martingale properties of the score function no longer hold. We will detail it shortly.

Following classical regret decompositions, the total regret splits into two components:
1) Exploration regret, governed by the choice of algorithm (EG, UCB, or TS); 2) Exploitation regret, arising from imperfect estimation of $\beta$ due to finite samples. The latter is more challenging in online survival analysis, as it requires constructing a confidence ellipsoid, i.e., a uniform bound on $\hat{\beta}^*_t - \beta$ over time. Let $\mathcal{F}_{\tau}$ denote the $\sigma$-algebra generated by ${\mathcal{D}_i}$ observed up to calendar time $\tau\in[0,\infty)$. Ideally, if the score process $U(\tau,\beta)$ in Equation~\eqref{eq:score_func} were a martingale in $\tau$, standard martingale concentration tools would apply. However, under staggered entry where two time scales $(\tau,s)$ are involved, and $U(\tau,\beta)$ is no longer a martingale due to the dependency of $\bar{X}(\tau,s)$ on $\tau$. An informal law of large numbers suggests that when $\mathcal{R}(\tau,s)$ is large, $\bar{X}(\tau,s)$ is approximately 
\begin{equation*}
    \mu(s) = \frac{\mathbb{E}\big[X\exp\{X^\top\beta\};Y\wedge C\geq s\big]}{\mathbb{E}\big[\exp\{X^\top\beta\};Y\wedge C\geq s\big]}.
\end{equation*}
Motivated by \cite{sellke1983sequential}, we instead define an adjusted score process
\begin{equation*}
    Q(\tau) = \sum_{i=1}^N \int_0^\tau [X_i-\mu(s)]\{N_i(\tau,ds)-A_i(\tau,ds)\},
\end{equation*}
where $A_i(\tau,s) = \Lambda_i\{s \wedge Y_i \wedge C_i \wedge (\tau-\tau_i)^+\}$ is the compensator, i.e., the expected event rate given the history. Unlike $U(\tau,\beta)$, $Q(\tau)$ is a martingale in $\tau$, allowing us to build concentration bounds. The confidence ellipsoid can then be constructed by bounding both $Q(\tau)$ and the deviation between $Q(\tau)$ and $U(\tau,\beta)$.

Before proceeding, we introduce the assumptions required to establish the regret upper bound.
\begin{assumption}\label{assump:1} (Boundedness)
\begin{itemize}
    \item[a.] $X_t(a)$ is uniformly bounded by a constant $L$, i.e., $\|X_t(a)\|_2\leq L$ for all $a\in\mathcal{A}$ and $t\in[T]$.  
    \item[b.] There exists a large constant $ Y_{\max}$, such that $|Y_t|\leq Y_{\max}$ uniformly holds for all $t\in[T]$.
    \item[c.] The minimum eigenvalue of $\tilde{I}_t$ satisfies $\lambda_{\min}(\tilde{I}_t) \geq \kappa t/d$ for some $\kappa>0$, where $\tilde{I}_t := \sum_{i=1}^N\int_0^{\tau_t} \{X_i-\mu(s)\}\{X_i-\mu(s)\}^\top A_i(\tau_t,ds)$ denotes the adjusted information matrix. 
\end{itemize}
\end{assumption}
Assumptions \ref{assump:1}.a and \ref{assump:1}.b are standard boundedness conditions commonly adopted in bandit and survival analysis settings \citep{chu2011contextual, agrawal2013thompson, filippi2010parametric}. Assumption \ref{assump:1}.c requires the information matrix to grow at least linearly in $t$, which is conceptually analogous to the \emph{persistence of excitation} condition and ensures sufficient exploration of the covariate space over time. A simple set of sufficient conditions (for intuition) is
\begin{enumerate}
    \item[(1)] $\mathbb{E}\!\left[\{X_i - \mu(s)\}\{X_i - \mu(s)\}^\top\right] \ge \lambda_0 I$,
    \item[(2)] $\Lambda(t) := \sum_{i=1}^n \int_0^t A_i(t, ds) \ge \eta t/d$,
\end{enumerate}
where $\Lambda(t)$ is the cumulative intensity.

Condition (1) guarantees that covariates within each risk set exhibit sufficient variability in all directions, which is analogous to the minimum-eigenvalue condition on the feature covariance matrix commonly assumed in the linear bandits literature \citep{abbasi2011improved, li2017provably, ding2021efficient, han2021generalized}. Condition (2) requires the expected number of failure events to grow linearly with time so that information accumulates, which is a requirement needed specifically in survival setting with right censoring. %When censoring is heavy and events are sparse, this condition may be violated, which is expected since online survival models only update upon observing failures.

More generally, if the accumulation of failure events grows sublinearly in reality, Condition (2) can be relaxed to $\Lambda(t) \ge \eta t^\alpha/d$ for any $\alpha > 0$, at the cost of a weaker regret bound of $\tilde{O}(\sqrt{d}\,T^{1-\alpha/2})$, which remains sublinear in $T$. A detailed discussion and proof are provided in Appendix~\ref{appendix:assump}.

We now state the confidence ellipsoid bound for online survival coefficient estimate.

\begin{theorem} (Confidence Ellipsoid) \label{thm:1}
Suppose Assumption~\ref{assump:1}.a holds, and $\Delta Q(t):=Q(\tau)-Q(\tau-)$ is conditionally sub-Gaussian. For any $\delta>0$, with probability at least $1-\delta$,
\begin{equation*}
    \|\hat{\beta}_t^*-\beta\|_{I_{t}} =\|U(\tau_t,\beta)\|_{I_t^{-1}} \leq   O(\sqrt{d\log (4tL^2/d)-2\log \delta})
\end{equation*}
holds uniformly for all $t\in[T]$. 
\end{theorem}
The proof of Theorem \ref{thm:1} is provided in Section \ref{appendix:thm1_proof} of the Supplementary Material.
Next, we establish the main theoretical result of the paper.

\begin{theorem} (Regret Bound)\label{thm:2}
    Under Assumption \ref{assump:1}, the regret upper bound for online survival analysis under either EG, UCB, or TS algorithm for exploration, is given by
    \begin{equation*}
        \text{Regret}(T) = \tilde{O}(d\sqrt{T}).
    \end{equation*}
\end{theorem}
\noindent\textbf{Remark.}
The regret order $\tilde{O}(\cdot)$ suppresses logarithmic factors in $T$. These logarithmic terms differ slightly among EG, UCB, and TS, with discrepancies of at most an additional $O(\sqrt{\log T})$. Detailed proofs are provided in Appendix~\ref{appendix:thm2_proof}. Note that the dimension $d$ implicitly incorporates the number of arms. In the special case where $\Phi(S,A) = (S^\top\boldsymbol{1}\{A=1\}, \dots, S^\top\boldsymbol{1}\{A=K\})^\top$, we have $d = d_0 K$, and the regret bound is thus $\tilde{O}(d_0 K \sqrt{T})$. Despite the additional challenges introduced by survival analysis, including right censoring and staggered entry, this regret rate is comparable to that of classical linear contextual bandits \citep{abbasi2011improved,li2017provably}.

\section{Online Analysis on SEER Dataset}

The SEER (Surveillance, Epidemiology, and End Results) database\footnote{\url{https://seer.cancer.gov/}} is a comprehensive, population-based cancer registry that collects and publishes data on cancer incidence, prevalence, survival, and treatment across diverse regions in the United States. It covers dozens of major cancer types, including breast, lung, and bladder cancer. In this study, we focus on patients with malignant breast cancer, which has been widely investigated in prior research \citep{chen2007prognostic,grover2017survival,lin2019incidence}.

We restrict our analysis to patients diagnosed with malignant neoplasms between 2004 and 2015. The final cohort consists of $290{,}091$ patients, characterized by covariates including age, marital status, CS tumor size, race, Estrogen Receptor (ER) status, and Progesterone Receptor (PR) status. Following primary data preprocessing and dummy-variable encoding, we define $S_i$ as the $8$-dimensional covariate vector for patient $i$. Eligible patients were treated with either no surgery, breast-conserving surgery (BCS), or mastectomy. Given the small proportion of patients who did not receive surgery, which can induce instability in bandit learning, we focus on a binary treatment setting with $\mathcal{A}=\{0,1\}$, where BCS is encoded as $A=0$ and mastectomy as $A=1$. Note that while this application considers a binary action space, our proposed methodology generalizes naturally to finite discrete action spaces.

For each patient, the SEER dataset provides both the year of diagnosis and the time from diagnosis to treatment. We sum these to determine the calendar time at which the treatment actually occurs, denoted by $\tau_i$. The dataset also reports the year of follow-up, which we use to compute the follow-up duration $C_i$ as the difference between the follow-up year and $\tau_i$. SEER records survival time in months, truncated by right censoring. For simplicity, we assume that both the year of diagnosis and the year of follow-up occur in January (the first month), so that all time measurements are expressed on a monthly scale.

Figure~\ref{fig:event_rate} illustrates the event rate in the SEER data. We observe that the cumulative number of events grows approximately as $O(t^{0.7})$, corresponding to $\alpha \approx 0.7$ in Condition (2) of the extension of Assumption~\ref{assump:1}.c, which in turn implies a regret upper bound of approximately $\tilde{O}(T^{0.65})$.

\begin{figure}[tbh]
    \centering
    \includegraphics[width=0.6\linewidth]{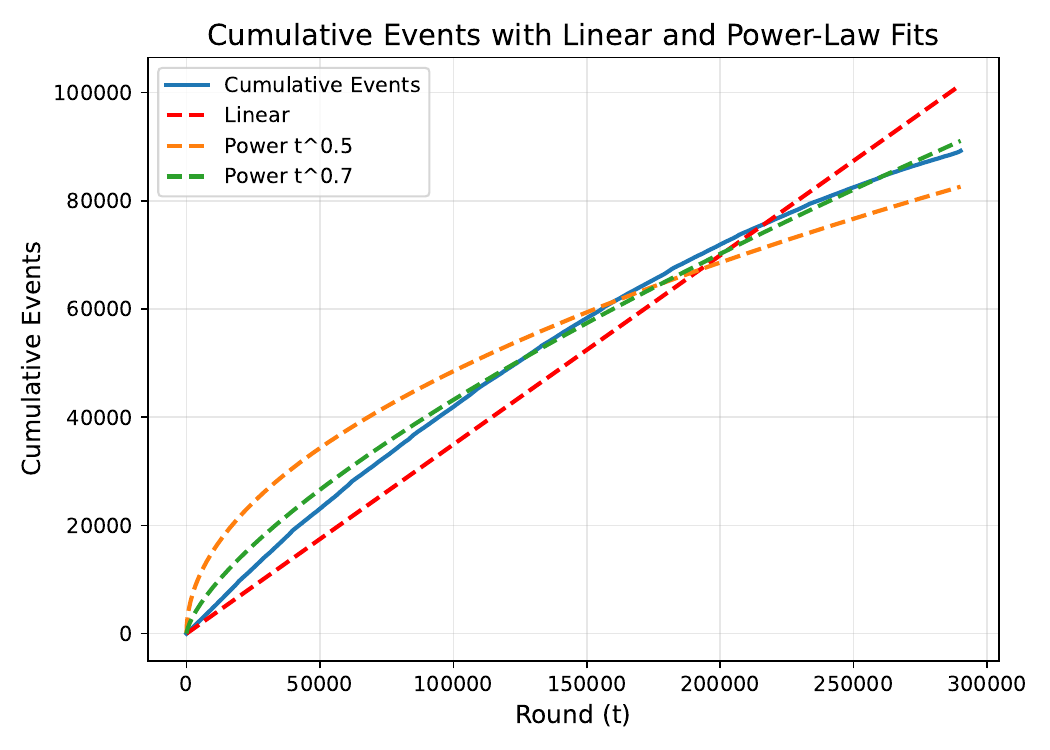}
    \caption{Event (death) rate over time in SEER data}
    \label{fig:event_rate}
\end{figure}

From 2004 to 2015, we obtain $T = 153$ rounds at a monthly resolution, excluding only a few months without new diagnoses. To evaluate bandit learning using logged observational data, we designate an initial exploration period consisting of the first several months during which the cumulative number of deaths across all patients is less than $n_0 = 500$. After this burn-in period, we apply one of three bandit algorithms: EG, UCB, or TS, to select an action $A_i$ for each patient $i$ observed in round $t$. Given the high data volume, thousands of patients may arrive in the same round, so the bandit procedure is naturally adapted to batch updates. %We retain a patient for training only if the bandit’s selected action matches the observed treatment. Otherwise, the patient is excluded due to unobserved counterfactual outcomes.

Upon completion of the bandit learning procedure, the retained dataset is used to fit a Cox PH model, from which we obtain an estimate of the regression coefficients $\beta$. We use this estimate as a reference model to define patient-specific optimal actions $A_i^{*}$ and the corresponding survival probabilities
\[
S(\tau_0\mid X) = \mathbb{P}(Y > \tau_0\mid X), \quad \tau_0 \in \{25, 50, 75, 100\},
\] 
under both the optimal action $A_i^{*}$ and the bandit-selected action ${A}_i$. We calculate the average of $\hat{S}(\tau_0)$ over all samples appeared from the beginning to the current stage (excluding a burning round of the first 3 months), with the resulting performance across algorithms summarized in Figure~\ref{fig:SEER_all}. %For brevity, we report only the average survival probabilities guided by bandit learning for $\tau_0 \in \{25, 100\}$ in the main paper. The complete set of results is provided in Figure~\ref{fig:SEER_all} in appendix.

\begin{figure*}[htb]
    \centering

    \begin{subfigure}{0.44\linewidth}
        \includegraphics[width=\linewidth]{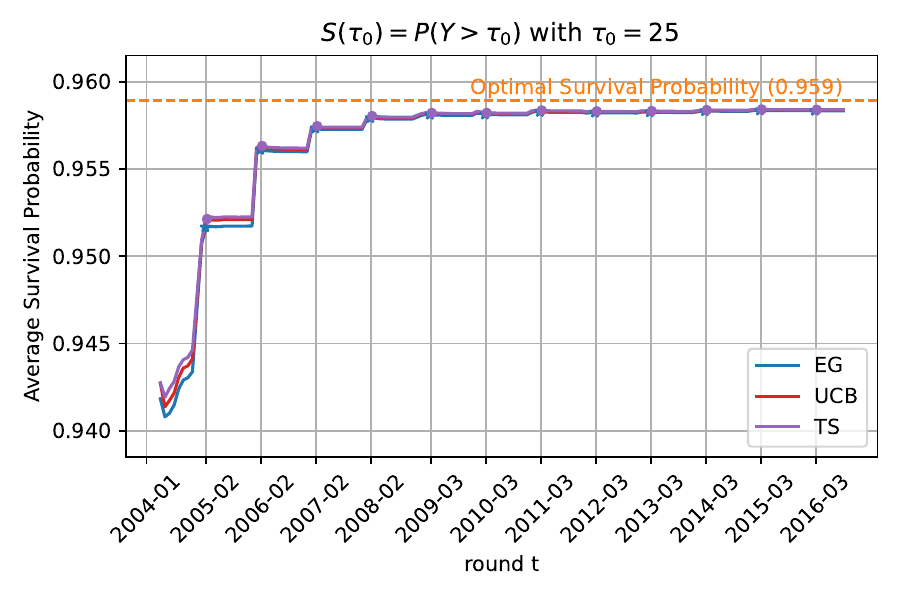}
        \caption{$\tau_0 = 25$}
    \end{subfigure}
    \hspace{0.07\textwidth}
    \begin{subfigure}{0.44\linewidth}
        \includegraphics[width=\linewidth]{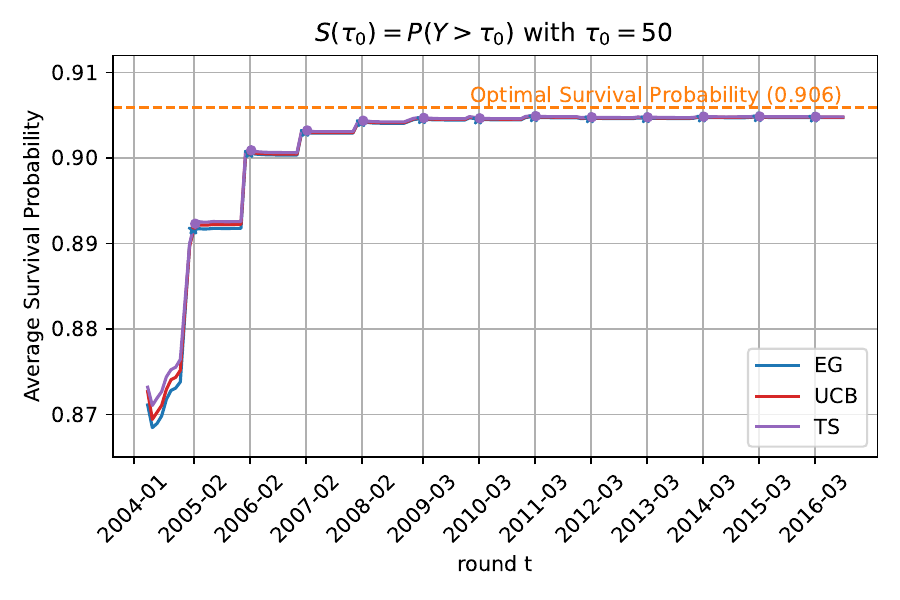}
        \caption{$\tau_0 = 50$}
    \end{subfigure}
    
    \medskip

    \begin{subfigure}{0.44\linewidth}
        \includegraphics[width=\linewidth]{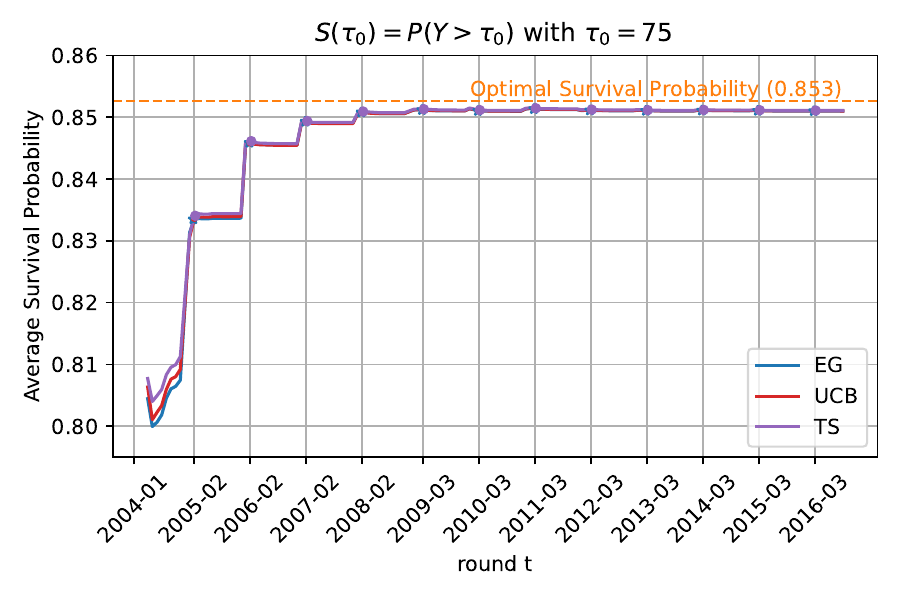}
        \caption{$\tau_0 = 75$}
    \end{subfigure}
    \hspace{0.07\textwidth}
    \begin{subfigure}{0.44\linewidth}
        \includegraphics[width=\linewidth]{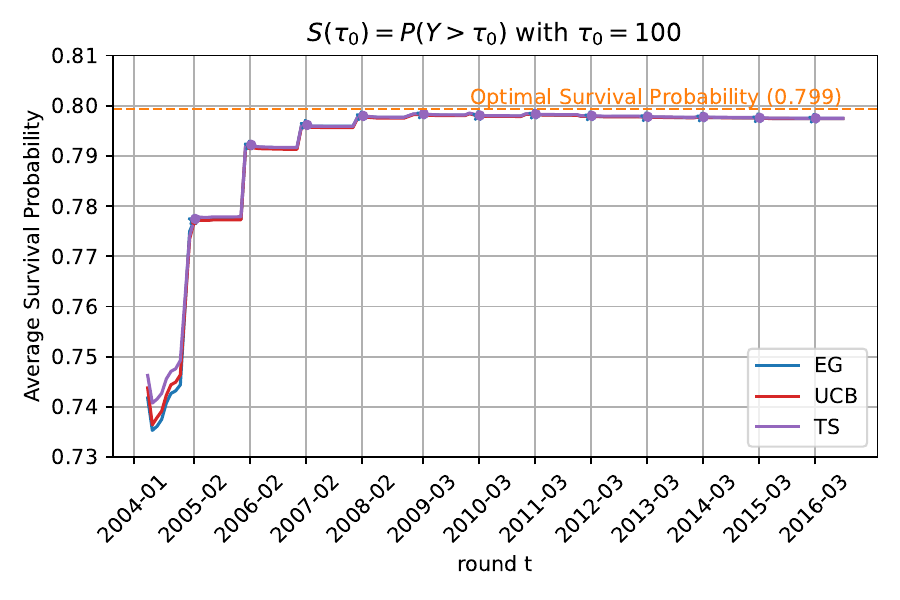}
        \caption{$\tau_0 = 100$}
    \end{subfigure}

    \caption{The average survival probability $S(\tau_0)$ under different $\tau_0 \in \{25, 50, 75, 100\}$ for our proposed algorithms.}
    \label{fig:SEER_all}
\end{figure*}

As shown in the figure, all algorithms quickly learn the underlying survival patterns as time progresses and converge toward the optimal survival probability $S(\tau_0)$. The large fluctuations observed, particularly at the beginning of each year, are primarily driven by the sudden increase in the number of available samples concentrated in specific months. It is important to note that this pattern may not perfectly reflect reality, since the exact month of diagnosis is not recorded in the version of the SEER dataset we used. Overall, the bandit learning algorithms exhibit a consistent and rapid convergence, demonstrating the practical applicability of our approach to online survival analysis scenarios.

Although the Cox proportional hazards model is widely used, the true data-generating process may deviate from the proportional hazards assumption. To assess robustness under model misspecification, we fit a Random Survival Forest (RSF) model \citep{ishwaran2008random} on the SEER data and treat it as an alternative reference model. The results for $\tau_0 = 100$ are presented in Figure~\ref{fig:seer_model_misspecification}. While the bandit algorithms (EG, UCB, and TS) are still trained under the Cox PH framework, their performance is evaluated against survival probabilities derived from the RSF model. As shown in the figure, although a modest gap remains relative to the oracle policy based on RSF (which is consistent with model misspecification), the online survival bandit algorithms continue to improve the average survival probability over time. Nevertheless, the Cox PH-based updates remain stable and exhibit robust performance under misspecification.

\begin{figure}
    \centering
    \includegraphics[width=0.6\linewidth]{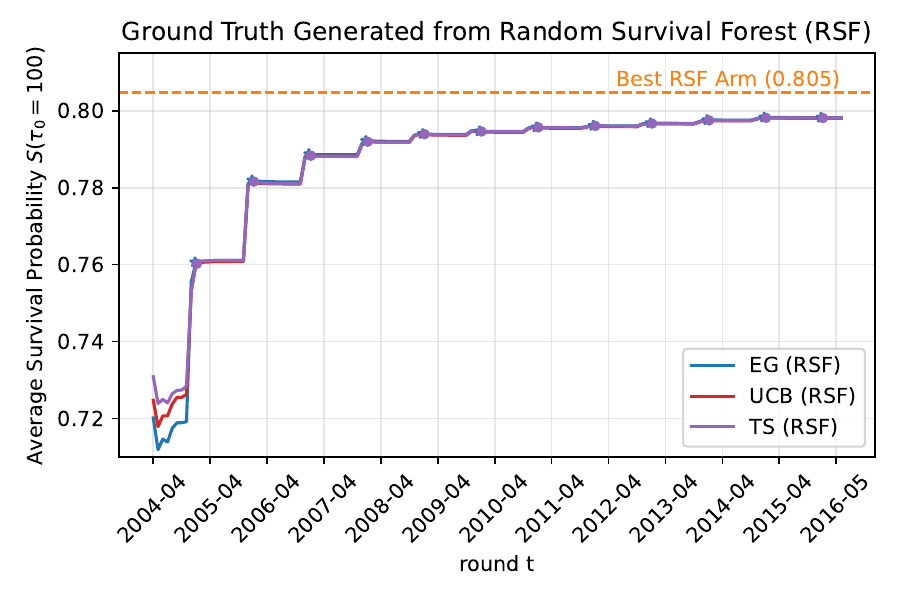}
    \caption{Performance on SEER data with ground truth generated from Random Survival Forest (RSF)}
    \label{fig:seer_model_misspecification}
\end{figure}

\section{More on Simulations}\label{sec:sim}

In this section, we present additional simulation results evaluating the performance of the proposed algorithm in approximating the true parameter $\beta$ and in driving the cumulative reward toward the optimum. As work in online bandit frameworks for survival data with right censoring remains limited, we focus primarily on comparing the three proposed exploration strategies: EG, UCB, and TS.

We also consider a naive baseline that, at each round, estimates $\hat{\beta}_t$ using only individuals whose failure or censoring status has already been observed. This approach is inherently biased, as individuals with shorter survival times are more likely to have their outcomes observed earlier. For completeness, the results for this baseline are reported in Appendix~\ref{appendix:naive}.

We consider $T=500$ rounds, where the gap time between two consecutive rounds follows a $\text{Poisson}(\lambda)$ distribution with $\lambda=1$. At each round $t$, a covariate vector $S_t = (S_{t,0}, S_{t,1}, S_{t,2})^\top \in \mathbb{R}^3$ is generated with $S_{t,0} \sim \text{Unif}(1,4)$, $S_{t,1} \sim \mathcal{N}(3,1)$, and $S_{t,2} \sim \mathcal{N}(2,1)$.

We consider a binary action space $\mathcal{A}=\{0,1\}$ and define the feature mapping as $X_t = \Phi(S_t,A_t) = \big(S_t^\top \boldsymbol{1}\{A_t=0\}, \; S_t^\top \boldsymbol{1}\{A_t=1\}\big)^\top$. The censoring time is generated from an exponential distribution, $C_t \sim \text{Exp}(\tau)$ with $\tau=5$. We define the true hazard function as $\lambda(\tau \mid X) = \lambda_0(\tau)\exp(X^\top \beta)$ where $\lambda_0(\tau) = 1$, and the true coefficient vector $\beta = (0.5, -0.3, -0.2, \; 0.2, 0.6, -0.1)^\top$.
The survival time is then generated by
\[
Y_t = \Lambda_0^{-1}\!\left(-\frac{\log(U)}{\exp(X^\top\beta)}\right), 
\quad U \sim \text{Unif}(0,1),
\]
where $\Lambda_0(\tau) = \int_0^\tau \lambda_0(u)\,du = \tau$. The censoring indicator is thus generated by $\delta_t = \boldsymbol{1}\{Y_t 
\leq C_t\}$.

This data-generating process is repeated $B=100$ times. In each replication, we estimate $\hat{\beta}_t$ at each round, record both the algorithm-recommended action $A_t$ and the optimal action $A_t^*$, and compute (a) the mean squared error (MSE) of $\hat{\beta}_t^*$, and (b) the estimated average survival probability $\hat{S}(\tau_0) = \sum_t S(\tau_0\mid X_t)/T$ with $\tau_0=1$. The results are reported in Figure~\ref{fig:sim}, where the solid lines represent the mean across $B=100$ replications and the shaded areas denote confidence intervals.

\begin{figure*}[tbh]
    \centering

    \begin{subfigure}{0.46\linewidth}
        \includegraphics[width=\linewidth]{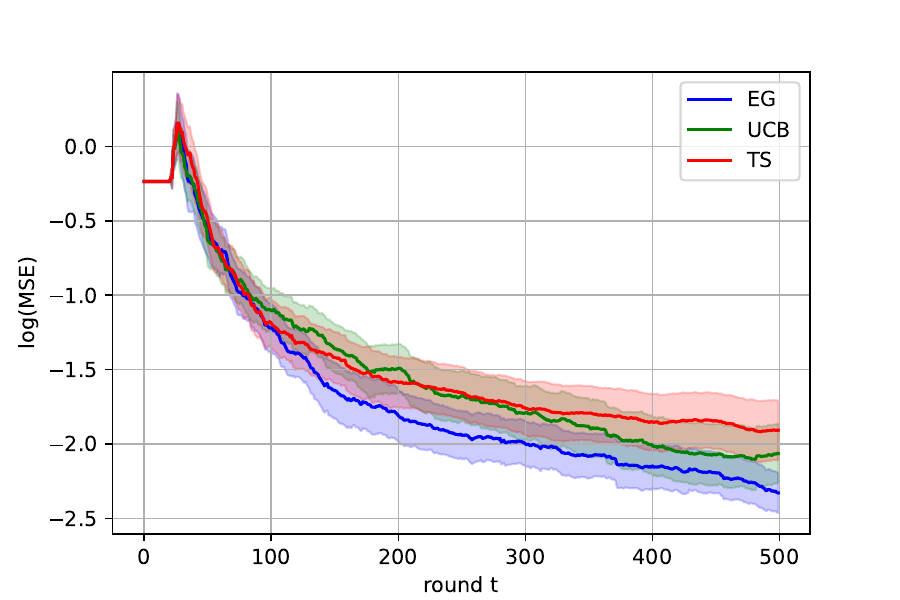}
        \caption{log(MSE) of $\hat{\beta}_t^*$}
    \end{subfigure}
    \hspace{0.07\textwidth}
    \begin{subfigure}{0.43\linewidth}
        \includegraphics[width=\linewidth]{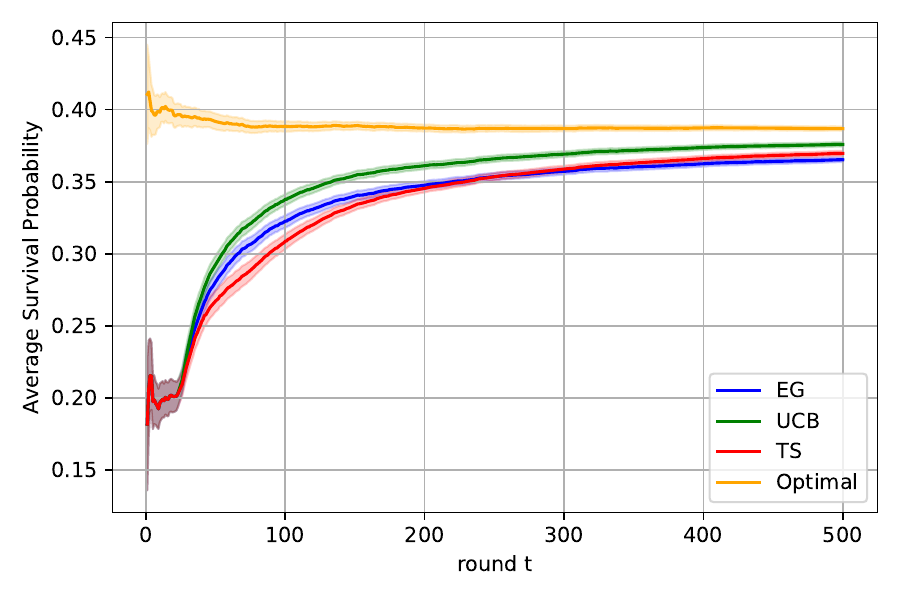}
        \caption{Average survival probability}
    \end{subfigure}

    \caption{The estimation performance for EG-, UCB-, and TS-based online survival bandit algorithms.}
    \label{fig:sim}
\end{figure*}

As shown in the figure, all algorithms converge rapidly: the estimates $\hat{\beta}_t^*$ approach the true coefficients $\beta$ quickly, and the average survival probabilities approach the optimal values at a fast rate. These results validate the effectiveness of our proposed method in achieving fast convergence in an online bandit setting with survival data. The code supporting all simulations and experiments in this paper is available at \href{https://github.com/YangXU63/SurvivalBandits}{this GitHub repository}.
Additional simulation results from sensitivity analyses examining varying censoring rates and model misspecification are provided in Appendix~\ref{appendix:sensitivity}.

\section{Summary and Future Work}
In this paper, we proposed a systematic bandit framework for online survival analysis, encompassing EG, UCB, and TS algorithms for exploration. We established strong theoretical guarantees by deriving sublinear regret bounds. Our empirical evaluation, based on both simulations and a semi-real dataset from SEER, demonstrated that the proposed framework can efficiently learn the underlying model, adapt to the data, and balance exploration and exploitation to quickly approach optimal performance.

Several promising directions remain for future work. One is to extend beyond the Cox PH model to accommodate more flexible survival modeling needs, for example by integrating neural networks through generalizations such as DeepSurv \citep{katzman2018deepsurv}. Another is to investigate more efficient approximation methods for online coefficient updates. Due to staggered entry, the risk set and log-likelihood from the previous round can be reused to incorporate new information for the next round, but approximations that further accelerate this process would be valuable to explore.

\section{Impact Statement}
This paper studies how survival analysis, a widely used methodology in areas such as precision medicine, clinical trials, customer churn analysis, and inventory management, can be adapted to online decision-making settings. The proposed contextual bandit framework supports adaptive and timely decisions under censoring by explicitly accounting for delayed and incomplete outcomes. By enabling principled learning from censored data in sequential settings, this work has the potential to improve decision support across a broad range of applications.

\bibliographystyle{unsrtnat}
\bibliography{reference} 

\newpage
\appendix
\begin{center}
    \textbf{\LARGE Appendix}
\end{center}

\section{Detailed Discussion of Assumption \ref{assump:1}.c and Its Relationship to the Regret Bound}
\label{appendix:assump}

\subsection{A Sufficient Condition for Assumption \ref{assump:1}.c}

As briefly discussed in the main paper, a simple set of sufficient conditions for Assumption~\ref{assump:1}.c is
\begin{enumerate}
    \item[(1)] $\mathbb{E}\!\left[(X_i - \mu(s))(X_i - \mu(s))^\top\right] \succeq \lambda_0 I$,
    \item[(2)] $\Lambda(t) := \sum_{i=1}^n \int_0^t A_i(t, ds) \ge \eta t / d$,
\end{enumerate}
where $\Lambda(t)$ denotes the cumulative intensity, that is, the expected number of observed events up to round $t$. Under these conditions, we have
\begin{equation*}
    \tilde{I}_t \gtrsim (\lambda_0 I)(\eta t / d)
    = (\lambda_0 \eta)\, t I / d
    := \kappa t I / d.
\end{equation*}

We provide further intuition for the two sufficient conditions below.

Condition (1) ensures that covariates within each risk set exhibit sufficient variability in all directions. This is analogous to the minimum-eigenvalue condition on the feature covariance matrix commonly assumed in the linear bandits literature \citep{abbasi2011improved, li2017provably, ding2021efficient, han2021generalized}.

Condition (2) arises naturally from the survival analysis setting. It requires that uncensored events accumulate at a positive rate in $t$. In the main paper, we adopt linear growth as the canonical case, mirroring classical linear contextual bandits where, in the absence of censoring, information accumulates at every round and Condition (1) alone is typically sufficient. In contrast, in survival settings, heavy right censoring may violate this linear growth assumption by substantially slowing the accumulation of observed events. To account for this, we next consider a relaxation of Assumption~\ref{assump:1}.c.

\subsection{Relaxation of Assumption \ref{assump:1}.c}

Even in scenarios where the assumption of linearly growing event counts does not hold, one can consider a more general condition of the form
\[
\lambda_{\min}(\tilde{I}_t) \ge \kappa t^{\alpha}, \quad \alpha \in (0,1).
\]
Under this relaxed assumption, information accumulates at a slower rate. Nevertheless, it is essential that some positive rate of event accumulation is maintained, namely $\alpha > 0$, so that $\tilde{I}_t$ continues to grow and the bandit updates remain nontrivial. This relaxation therefore preserves the learnability of the problem while allowing for substantially sparser feedback.

By extending the arguments and proofs in the appendix to accommodate this relaxed condition, we obtain
\[
\sum_{t=1}^T \|X_t(a)\|_{I_t^{-1}}^2
= O\!\left(\frac{L^2}{\kappa} T^{1-\alpha}\right).
\]
As a result, the regret upper bound (for EG, UCB, and TS) becomes
\[
\tilde{O}\!\left(\sqrt{d \log T} \cdot \sqrt{T \cdot T^{1-\alpha}}\right)
= \tilde{O}\!\left(\sqrt{d}\, T^{1-\alpha/2}\right).
\]
This rate remains sublinear in $T$, though it is larger than the classical $\tilde{O}(\sqrt{T})$ bound due to limited information accumulation caused by censoring and the inherent structure of survival data. This analysis highlights a fundamental tradeoff specific to the survival setting: achieving smaller regret bounds requires stronger guarantees on the growth of information in the Gram matrix. Our formulation of Assumption~\ref{assump:1}.c is sufficiently flexible to accommodate this tradeoff across a wide range of censoring regimes.

\section{Sensitivity Analysis}\label{appendix:more_simulation}

\subsection{Variation of censoring rate on finite-sample performance}\label{appendix:sensitivity}

Heavy censoring leads to fewer observed failures and thus a poorer performance on finite samples. This loss of information can introduce bias \citep{struthers1986misspecified}, inflate variance, and weaken the accuracy of asymptotic approximations \citep{kalbfleisch2002statistical}.

To further test the variation of censoring rate on our algorithm performance, we adjust the value of $\tau$ (which controls the censoring time at round $t$ by $C_t\sim \text{Exp}(\tau)$) in a range of $\tau\in\{0.1, 0.5, 1, 2, 5, 10\}$. The corresponding censoring rates are $\{0.86,0.646,0.552,0.452,0.304,0.242 \}$. We then calculate the $\beta$ estimation MSE and average survival probability change under different $\tau$s. Results summarized in the table below.

\begin{table}[h!]
\centering
\begin{tabular}{lcccccc}
\hline
 & $\tau = 0.1$ & $\tau = 0.5$ & $\tau = 1$ & $\tau = 2$ & $\tau = 5$ & $\tau = 10$ \\
\hline
MSE of $\hat{\beta}$ & 0.0692 & 0.0302 & 0.0208 & 0.0162 & 0.0145 & 0.0135 \\
Mean Survival Probability & 0.6144 & 0.6298 & 0.6325 & 0.6338 & 0.6349 & 0.6356 \\
\hline
\end{tabular}
\end{table}

We can see that a larger $\tau$ leads to lower censoring, which makes the $\beta$ estimates closer to their true values and improves the bandit's performance, resulting in a slightly higher mean survival probability by the end of $T = 500$ rounds. This is consistent with intuition: when convergence is delayed, the algorithm needs more time to return to reasonable recommendations, which can mildly decrease the average survival probability.

In practice, a useful rule of thumb is the standard “events-per-variable” (EPV) criterion \citep{peduzzi1996simulation}. Following this, we recommend starting online learning only once each action group has collected more than $10d_0$ events, with $d_0$ denoting the dimension of the state variable.

\subsection{Variation of model misspecification on finite-sample performance} \label{appendix:model_misspecification}

In this section, we provide simulation results under model misspecification, while keeping the bandit learning procedure based on the Cox proportional hazards (PH) model. Let $X$ denote the constructed feature vector and $U \sim \text{Uniform}(0,1)$. In all settings, survival times are generated via an inverse-transform representation of the form
$Y = F^{-1}(U \mid X)$,
followed by right censoring $R = \min(Y, C)$ and the event indicator $\delta = \mathbf{1}(Y \leq C)$.

\begin{figure*}[htb]
    \centering
    \begin{subfigure}{0.44\linewidth}
        \includegraphics[width=\linewidth]{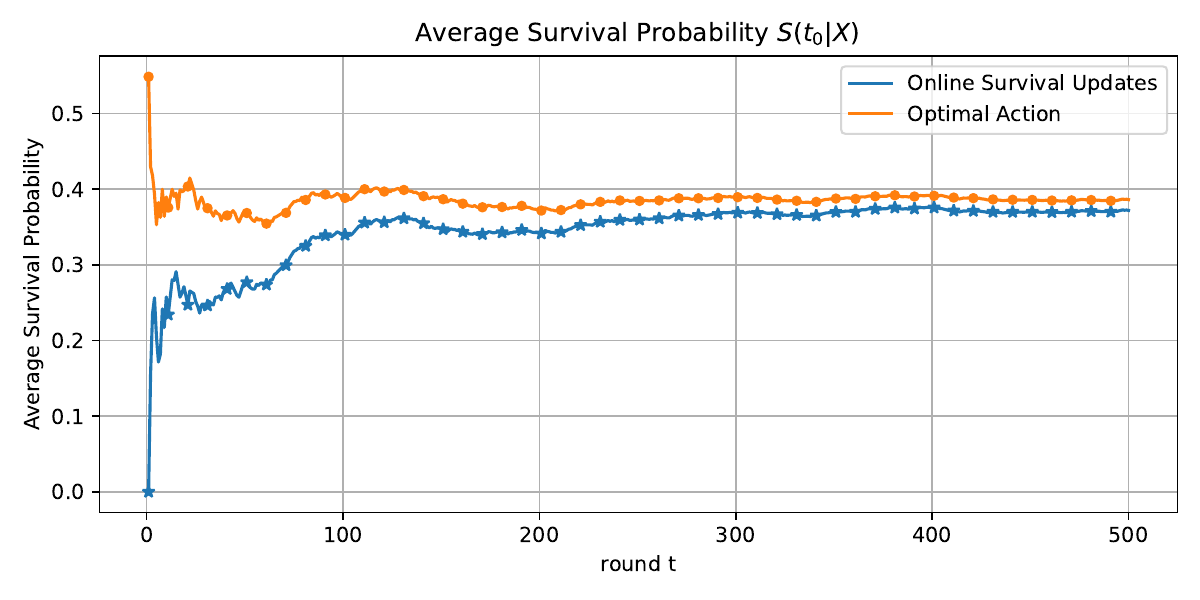}
        \caption{Ground truth generated from Cox PH model}
    \end{subfigure}
    \hspace{0.07\textwidth}
    \begin{subfigure}{0.44\linewidth}
        \includegraphics[width=\linewidth]{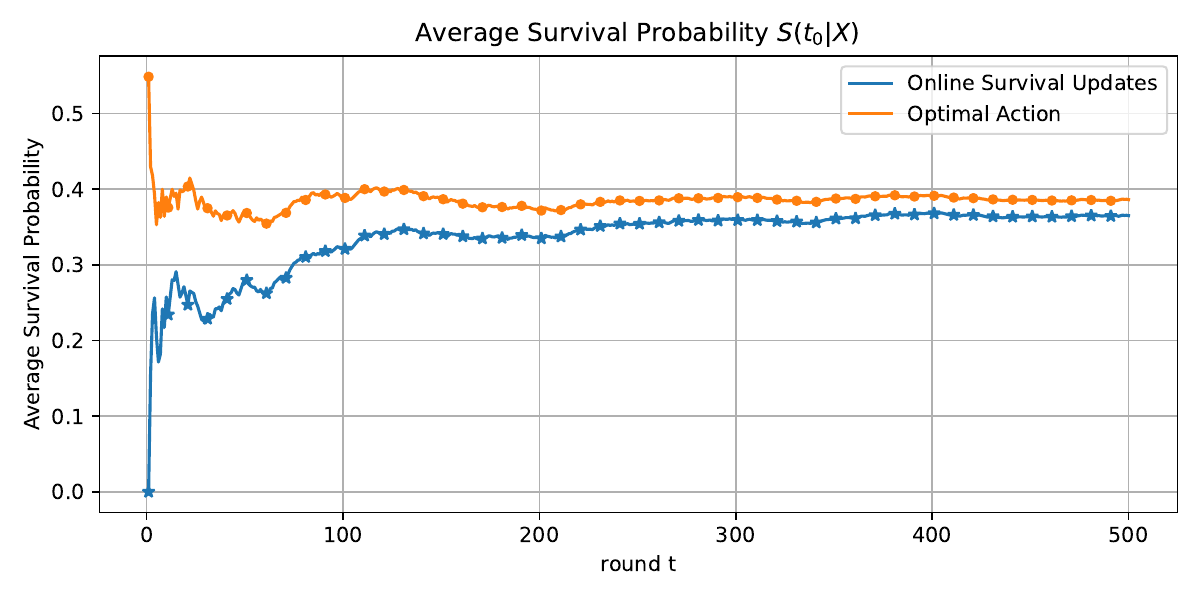}
        \caption{Ground truth generated from disturbed Cox PH model}
    \end{subfigure}
    \medskip
    \begin{subfigure}{0.44\linewidth}
        \includegraphics[width=\linewidth]{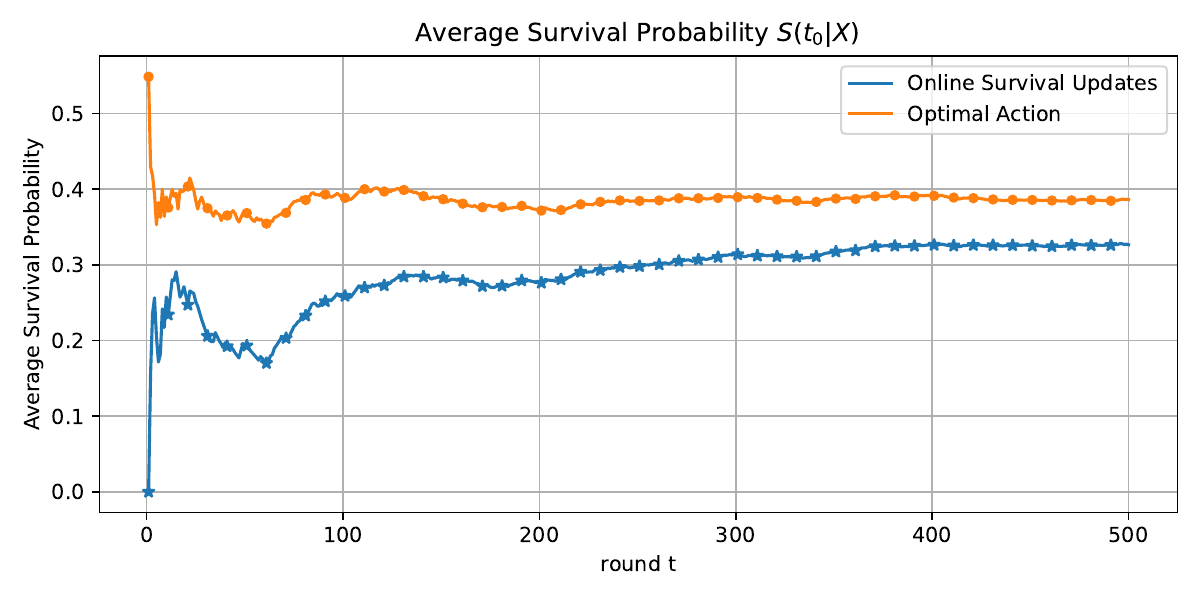}
        \caption{Ground truth generated from AFT model}
    \end{subfigure}
    \hspace{0.07\textwidth}
    \begin{subfigure}{0.44\linewidth}
        \includegraphics[width=\linewidth]{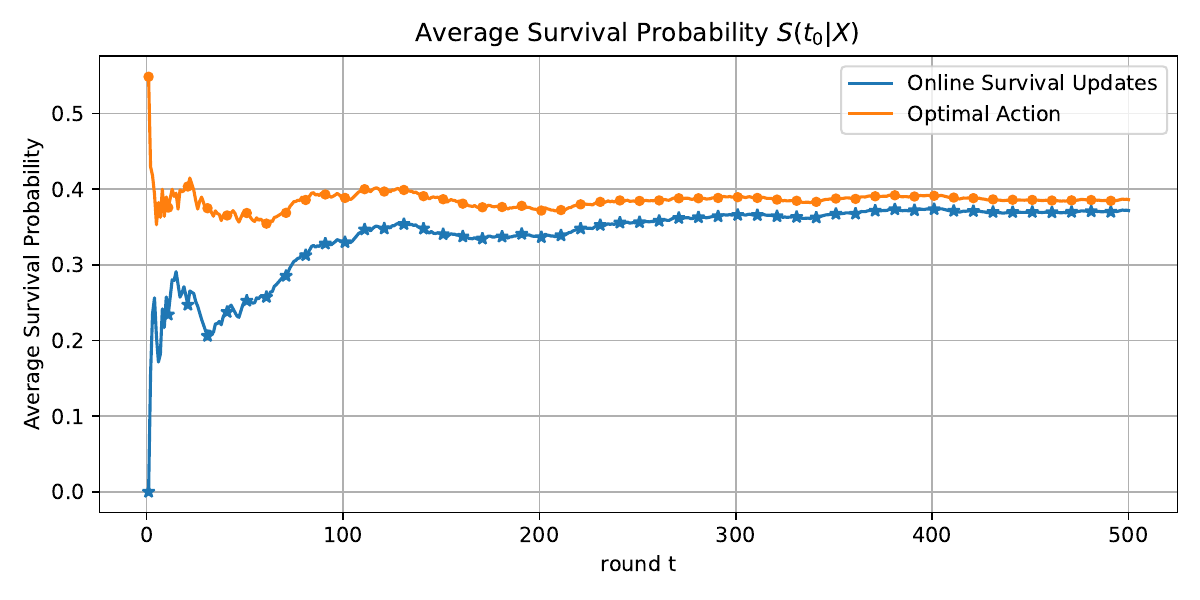}
        \caption{Ground truth generated from piecewise KM-type model}
    \end{subfigure}
    \caption{The average survival probability }
    \label{fig:sim_model_misspecification}
\end{figure*}
\textbf{(a) Correctly specified Cox PH model.}
The correctly specified case for reference, which follows the same Cox PH form as the setup in Section \ref{sec:sim} of the main paper:
$Y = -{\log U}/{\exp(X^\top \beta)}$.

\textbf{(b) Disturbed Cox PH model.}
We introduce additive noise in the log-hazard:
$Y = -{\log U}/{\exp(X^\top \beta + \varepsilon)}$, where $ \varepsilon \sim \mathcal{N}(0, 5^2)$.
This preserves the PH structure but induces random perturbations in the hazard rate.

\textbf{(c) Accelerated Failure Time (AFT) model.}
We consider a log-linear survival model:
$\log Y = X^\top \beta + \varepsilon, \quad \varepsilon \sim \mathcal{N}(0, \sigma^2)$,
equivalently,
$Y = \exp(X^\top \beta + \varepsilon)$.

\textbf{(d) Piecewise constant hazard model.}
We consider a misspecified baseline hazard with random regime shifts:
$Y = -\log U / (h_0 \exp(X^\top \beta))$, where $h_0 \in \{0.5, 1.0, 2.0\}$ and is sampled from a discrete uniform distribution.

\textbf{Evaluation.}
Figure~\ref{fig:sim_model_misspecification} reports the average survival probability under each setting. In all cases, the same Cox PH-based bandit algorithms (EG, UCB, TS) are applied for action selection, and performance is evaluated using the true survival outcomes generated from the corresponding DGP.

Across all misspecified settings (b)-(d), we observe a moderate degradation relative to the correctly specified case (a). However, the proposed algorithms consistently improve over time and remain stable, demonstrating robustness to deviations from the Cox PH assumption.

\subsection{Comparison with Naive Method}\label{appendix:naive}

At round $t$, one can define a naive method as follows: the learner uses only individuals whose failure/censoring status is already observed and discards those whose outcomes remain pending. This approach is inherently biased because individuals with shorter survival times are more likely to have their outcomes observed earlier, creating a selection bias. Our online survival framework explicitly addresses this issue by incorporating censoring and staggered entry.  

Below is the comparison table summarizing the performance of our method versus the naive alternative at $t \in \{0,100,200,300,400,500\}$, showing the MSE of the $\beta$ estimator (mean with standard error in parentheses), averaged over 100 replications:

\begin{center}
\begin{tabular}{c|c|c}
\hline
$t$ & naive & our method \\
\hline
0   & 0.7900 (0.0000) & 0.7900 (0.0000) \\
100  & 0.7862 (0.0382) & 0.7833 (0.0666) \\
200 & 0.4373 (0.3913) & 0.1362 (0.1048) \\
300 & 0.2721 (0.1504) & 0.1605 (0.3022) \\
400 & 0.2385 (0.1227) & 0.1075 (0.1966) \\
500 & 0.2145 (0.1053) & 0.0843 (0.1887) \\
\hline
\end{tabular}
\end{center}
As expected, the naive alternative consistently exhibits larger MSE than our method, highlighting the importance of explicitly accounting for right censoring in the online bandit framework.

\subsection{Average runtime of online survival updates over time}

\begin{figure}[htb]
    \centering
    \includegraphics[width=0.6\linewidth]{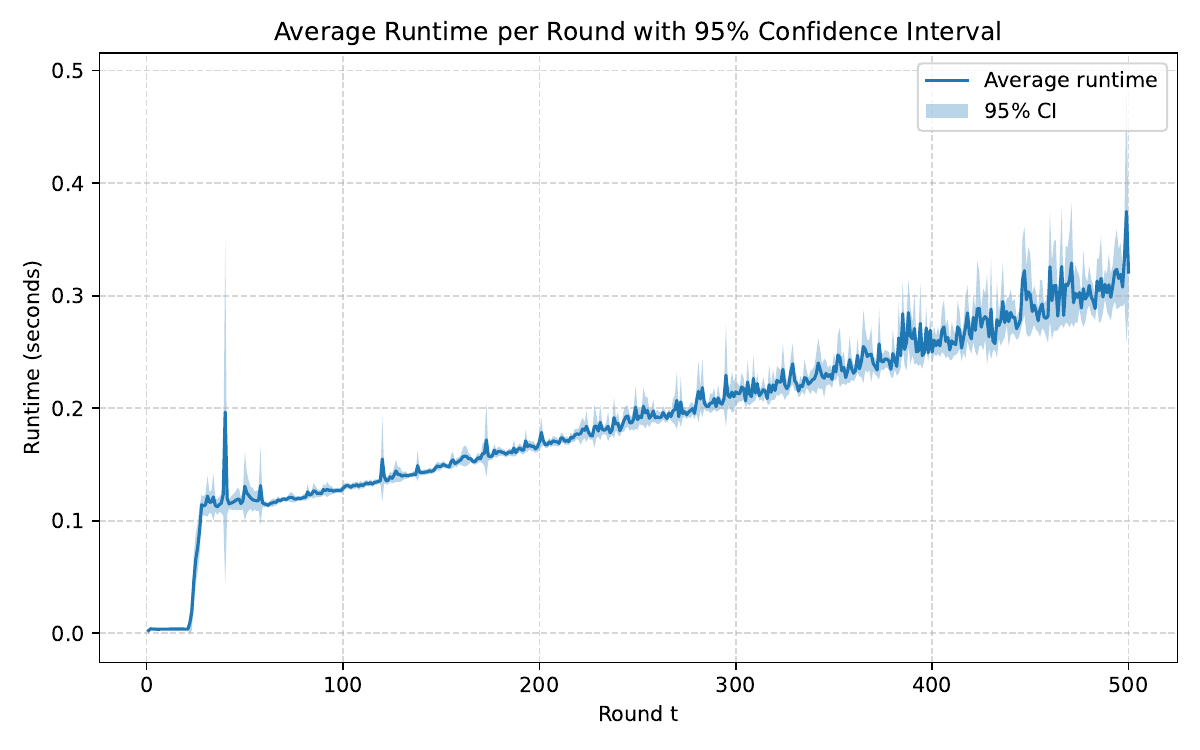}
    \caption{Average runtime of online survival bandits over time}
    \label{fig:sim_runtime}
\end{figure}

\section{Proof of Theorem \ref{thm:1}}\label{appendix:thm1_proof}
\begin{theorem} (Confidence Ellipsoid) 
Suppose Assumption~\ref{assump:1}.a holds, and $\Delta Q(t):=Q(\tau)-Q(\tau-)$ is conditionally sub-Gaussian. For any $\delta>0$, with probability at least $1-\delta$,
\begin{equation*}
    \|\hat{\beta}_t^*-\beta\|_{I_{t}} =\|U(\tau_t,\beta)\|_{I_t^{-1}} \leq   O\big(\sqrt{d\log (4tL^2)-\log \delta}\big)
\end{equation*}
holds uniformly for all $t\in[T]$. 
\end{theorem}

\textbf{Proof:}

Recall that the score process is defined as 
$$
U(\tau,\beta)= \sum_{i=1}^N \int_{0}^\tau\left[X_i-\bar{X}(\tau,s)\right]N_i(\tau,ds),
$$
where $\bar{X}(\tau,s)=\frac{\sum_{j\in\mathcal{R}(\tau,s)}\exp(X_j^\top\beta)\cdot X_j}{\sum_{j\in\mathcal{R}(\tau,s)}\exp(X_j^\top\beta)}$. At round $t$, we estimate $\hat{\beta}^*_t$ by solving $U(\tau_t,\hat{\beta}^*_t) = 0$.

By doing Taylor expansion for $U(\tau_t,\hat{\beta}^*_t)$ at the true value $\beta$, we have
$$
U(\tau_t,\hat{\beta}^*_t) = 0 = U(\tau_t,\beta) + I_t(\hat{\beta}^*_t-\beta) + O_p(\|\hat{\beta}^*_t-\beta\|^2),
$$
where $I_t = -\nabla^2l(\beta)$ is the information process, and $O_p(\|\hat{\beta}^*_t-\beta\|^2)$ is a negligible remainder term. Thus,
\begin{equation*}
    \hat{\beta}^*_t-\beta \approx -I_t^{-1}U(\tau_t,\beta).
\end{equation*}
According to the definition of matrix norm, 
$$
\|\hat{\beta}^*_t-\beta\|_{I_{t}} = \sqrt{(\hat{\beta}^*_t-\beta)^\top  I_t^{-1} (\hat{\beta}^*_t-\beta)} = \|I_t^{-1/2}U(\tau_t,\beta)\|_2 = \|U(\tau_t,\beta)\|_{I_t^{-1}}.
$$
Therefore, finding the high probability uniform bound for $\|\hat{\beta}^*_t-\beta\|_{I_{t}}$ is equivalent to bound $\|U(\tau_t,\beta)\|_{I_t^{-1}}$.

Let $A_i(\tau,ds) = \boldsymbol{1}_{\{i\in\mathcal{R}(\tau,s)\}}\Lambda_i(ds)$, so that $A_i(t,s) = \Lambda_i\{s\wedge T_i\wedge C_i\wedge (\tau-\tau_i)^+\}$. $A_i(\tau,ds)$ is the compensator, which captures the expected event rate given the history.
Since $U(\tau,\beta)$ denotes the score function accumulated up to calendar time $\tau$, we generalize it to a score process $U(\tau,s,\beta)$, where
$$
U(\tau,s,\beta) = \sum_{i=1}^N \int_{0}^s\left[X_i-\bar{X}(\tau,u)\right]\{N_i(\tau,du)-A_i(\tau,du)\}.
$$
By definition, $U(\tau,\tau,\beta) = U(\tau,\beta)$.

Let $\mathcal{F}_{\tau,s}$ be the sub-$\sigma$-algebra of $\mathcal{F}_s$ containing events which have been observed by calendar time $\tau$ and are of survival time $\leq s$. As proved in \cite{sellke1983sequential}, $N_i(\tau,s)-A_i(\tau,s)$ is an $\mathcal{F}_{\tau,s}$ martingale in $s$ for each fixed $\tau$. Thus, when $\tau$ is fixed, $\{U(\tau,s,\beta),\mathcal{F}_{\tau,s}\}$ is a martingale in $s$.

However, $U(\tau,\beta)$ is not a martingale in $t$ in general due to the dependence of $\bar{X}(\tau,s)$ on $\tau$. Fortunately, as briefly discussed in the main paper, by an informal law of large numbers, $\bar{X}(\tau,s)$ can be approximated by 
\begin{equation*}
    \mu(s) = \frac{\mathbb{E}\big[X\exp\{X^\top\beta\};Y\wedge C\geq s\big]}{\mathbb{E}\big[\exp\{X^\top\beta\};Y\wedge C\geq s\big]}
\end{equation*}
when $\mathcal{R}(\tau,s)$ is large. Therefore, we define 
\begin{equation*}
    Q(\tau) = \sum_{i=1}^N \int_0^\tau [X_i-\mu(s)]\{N_i(\tau,ds)-A_i(\tau,ds)\},
\end{equation*}
so that $\mu(s)$ is no longer a function of $\tau$.
Therefore, $\{Q(\tau),\mathcal{F}_{\tau}\}$ is a martingale in $\tau$.
To accommodate the two time scales, we further define $N(\tau,s) =\sum_{i=1}^N N_i(\tau,s)$, $A(\tau,s) =\sum_{i=1}^N A_i(\tau,s)$, and $D(\tau) = \mathbb{E}[N(\tau)]=\mathbb{E}[\sum_iN_i(\tau,\tau)]$, which represents the expected number of events (deaths) that occur before time $\tau$. Then
\begin{equation*}
    \begin{aligned}
    r(\tau) := U(\tau,\beta) - Q(\tau) 
    &= \sum_{i=1}^N \int_0^{\tau} [\bar{X}(\tau,s)-\mu(s)]\{N_i(\tau,ds)-A_i(\tau,ds)\}\\
    &=\int_0^{\tau} [\bar{X}(\tau,s)-\mu(s)]\{N(\tau,ds)-A(\tau,ds)\}.
    \end{aligned}
\end{equation*}
\cite{sellke1983sequential} proved in Theorem 1 of their paper that $r(\tau)$ is small compared to $D(\tau)$ uniformly in $\tau$. Hence, for any $0<\epsilon<1/15$, and uniformly over the arrival process,
\begin{equation*}
\mathbb{P}(|r(\tau)|\leq \kappa + D(\tau)^{0.5-\epsilon}, \tau \geq 0)\rightarrow 1
\end{equation*}
as $\kappa\rightarrow \infty$. Equivalently, we have $r(t) = O_p(1) + O_p(D(t)^{0.5-\epsilon})$.

To derive a uniform high-probability upper bound for $\|U(\tau,\beta)\|_{I_t^{-1}}$, it is sufficient to derive upper bounds for $\|Q(\tau)\|_{I_t^{-1}}$ and $\|r(\tau)\|_{I_t^{-1}}$ by triangle inequality:
$$
\|U(\tau_t,\beta)\|_{I_t^{-1}} = \|Q(\tau_t) + r(\tau_t)\|_{I_t^{-1}}\leq  \|Q(\tau_t)\|_{I_t^{-1}} + \|r(\tau_t)\|_{I_t^{-1}}.
$$
In the following steps, we first derive uniform upper bounds for $\|r(\tau_t)\|_{I_t^{-1}}$ and $ \|Q(\tau_t)\|_{I_t^{-1}}$ in Steps 1 and 2, and then conclude the proof of the theorem in Step 3.

\textbf{Step 1.} Let’s first handle $\|r(\tau_t)\|_{I_t^{-1}}$. Since 
$$
\begin{aligned}
I_t  &=\sum_{i=1}^N\int_0^{\tau_t} \sum_{j\in\mathcal{R}(\tau_t,s)}\omega_j(\tau_t,s)\left[(X_j-\bar{X}(\tau_t,s))(X_j-\bar{X}(\tau_t,s))^\top\right]N_i(\tau_t,ds),
\end{aligned}
$$
which is a sum of one weighted covariance matrix per event. By Assumption \ref{assump:1}.a, $X(a)$ is uniformly bounded by $L$, so $\|I_t\| =O(D(\tau_t))$. Therefore, the remainder term has an order
\begin{equation}\label{eq:r_1}
\begin{aligned}
\|r(\tau_t)\|_{I_t^{-1}} &\sim O(\|r(\tau_t)\|)\cdot O(\|I_t\|^{-1/2})\sim O_p(D(\tau_t)^{0.5-\epsilon}+1)\cdot O(D(\tau_t)^{-1/2}) =o(1).
\end{aligned}  
\end{equation}

%Thus, $\|r(t)\|_{V_t^{-1}}$ is negligible compared with $\|Q(t)\|_{I_t^{-1}}$.

\textbf{Step 2.} Next, let’s derive the order of $\|Q(\tau_t)\|_{I_t^{-1}}$.

First, we define the variance of $Q(\tau)$ as $\tilde{I}_{\tau}$, where
$$
\begin{aligned}
\tilde{I}_{t} &= \sum_{i=1}^N\int_0^{\tau_t} \left[\{X_i-\mu(s)\}\{X_i-\mu(s)\}^\top\right]A_i(\tau_t,ds),
\end{aligned}
$$
which is asymptotically equivalent to $I_{t}$ when $\mathcal{R}(\tau_t,s)$ is large. Thus, we have $\|I_t-\tilde{I}_t\|=o_p(1)$.

For the purpose of the proof, we generalize the adjusted information matrix $\tilde{I}_t$, which is evaluated at each round $t$, to a continuous-time counterpart $\tilde{\mathcal{I}}_{\tau}$, defined over calendar time $\tau\geq 0$, where
$$
\tilde{\mathcal{I}}_{\tau} = \sum_{i=1}^N\int_0^{\tau} \left[\{X_i-\mu(s)\}\{X_i-\mu(s)\}^\top\right]A_i(\tau,ds).
$$
Let $\lambda\in\mathbb{R}^d$ be arbitrary and consider for any $\tau\geq 0$,
\begin{equation*}
    M_{\tau}^\lambda := \exp\Big( \lambda^\top Q(\tau) - \frac{1}{2}\|\lambda\|_{\tilde{\mathcal{I}}_{\tau}}^2\Big).
\end{equation*}

Define $\Delta Q(\tau) = Q(\tau) - Q(\tau-) $, and $\Delta \tilde{\mathcal{I}}_t = \tilde{\mathcal{I}}_{\tau}-\tilde{\mathcal{I}}_{\tau-}$. Then we can decompose $M_{\tau}^\lambda$ by
$$
M_t^\lambda = M_{\tau-}^\lambda \cdot \exp\Big( \lambda^\top \Delta Q(\tau) - \frac{1}{2}\lambda^\top \Delta{\tilde{\mathcal{I}}_t}\lambda\Big).
$$
Since $\Delta Q(t)$ is conditionally sub-Gaussian, we have 
$$
\mathbb{E}[\exp\{ \lambda^\top \Delta Q(\tau)\}\mid \mathcal{F}_{\tau-}]\leq \exp\Big(\frac{1}{2}\lambda^\top \Delta{\tilde{\mathcal{I}}_\tau}\lambda\Big).
$$
We claim that $M_t^\lambda$ is a supermartingale, since
$$
\mathbb{E}[M_\tau^\lambda\mid \mathcal{F}_{\tau-}] = M_{\tau-}^\lambda\cdot \mathbb{E}\bigg[\exp\Big(\lambda^\top\Delta Q(\tau)-\frac{1}{2}\lambda^\top \Delta{\tilde{\mathcal{I}}_t}\lambda\Big)\big| \mathcal{F}_{\tau-}\bigg]\leq M_{\tau-}^\lambda \leq M_0^\lambda =1,
$$
which gives us $\mathbb{E}[M_\tau^\lambda] = \mathbb{E}\big[\mathbb{E}[M_\tau^\lambda\mid \mathcal{F}_{\tau-}]\big]\leq \mathbb{E}[M_0^\lambda] = 1$.

Then, following Lemma 9 of \cite{abbasi2011improved} on the self-normalized bound for vector-valued martingales, we can obtain a similar result, up to minor modifications, that
$$
\begin{aligned}
M_\tau &= \int_{\mathbb{R}^d} M_\tau^\lambda f(\lambda) d\lambda = \int_{\mathbb{R}^d} \exp\Big(\lambda^\top\Delta Q(\tau)-\frac{1}{2}\lambda^\top \Delta{\tilde{\mathcal{I}}_\tau}\lambda\Big) f(\lambda) d\lambda\\
&= \int_{\mathbb{R}^d} \exp\Big(-\frac{1}{2}\|\lambda - \tilde{\mathcal{I}}_t^{-1} Q(\tau)\|_{\tilde{\mathcal{I}}_\tau}^2+\frac{1}{2}\|Q(\tau)\|_{\tilde{\mathcal{I}}_t^{-1}}^2\Big)  d\lambda\\
& = c(\tilde{\mathcal{I}}_\tau)\exp\Big(\frac{1}{2}\|Q(\tau)\|_{\tilde{\mathcal{I}}_\tau^{-1}}^2\Big),
\end{aligned}
$$
where $c(\tilde{\mathcal{I}}_\tau) = \sqrt{(2\pi)^d/\det(\tilde{\mathcal{I}}_\tau)}$, and we assume that there is no informative prior on $\lambda$, i.e. $f(\lambda) = 1$.
Now, from $\mathbb{E}[M_\tau^\lambda]\leq 1$, we have for any $\delta>0$ that

\begin{equation}\label{eq:Q_2}
\begin{aligned}
&\mathbb{P}\bigg[\|Q(\tau)\|^2_{\tilde{\mathcal{I}}_\tau^{-1}} > -2\log (\delta\cdot c(\tilde{\mathcal{I}}_\tau))\bigg] = \mathbb{P}\bigg[\delta\cdot c(\tilde{\mathcal{I}}_\tau)\exp\Big(\frac{1}{2}\|Q(\tau)\|_{\tilde{\mathcal{I}}_\tau^{-1}}^2\Big) > 1\bigg]\\&=\mathbb{P}(\delta M_\tau >1)\leq \mathbb{E}[\delta M_\tau] \leq \delta,
\end{aligned}
\end{equation}
where the second last inequality holds by Markov’s inequality.

Finally, we use a stopping time construction to construct a uniform upper bound for $\|Q(\tau)\|_{\tilde{\mathcal{I}}_\tau^{-1}}$. Define the bad event
\begin{equation*}
    B_{\tau}(\delta) = \Big\{\omega\in\Omega: \|Q(\tau)\|_{\tilde{\mathcal{I}}_\tau^{-1}}^2 >-2\log (\delta\cdot c(\tilde{\mathcal{I}}_\tau)\Big\}.
\end{equation*}
We are interested in bounding the probability that $\bigcup_{\tau\geq 0}B_\tau(\delta)$ happens. Define $\tilde{\tau}(\omega) = \min\{\tau\geq 0: \omega \in B_\tau(\delta)\}$, with the convention that $\min\phi = \infty$. Then, $\tilde{\tau}$ is a stopping time, and 
\begin{equation*}
    \bigcup_{\tau\geq 0}B_\tau(\delta) = \{\omega:\tilde{\tau}<\infty\}.
\end{equation*}
Thus, by Equation \eqref{eq:Q_2}, we have
\begin{equation*}
\begin{aligned}
    &\mathbb{P}\bigg[\bigcup_{\tau\geq 0}B_\tau(\delta)\bigg] = \mathbb{P}(\tilde{\tau}<\infty) = \mathbb{P}\bigg[\|Q(\tilde{\tau})\|^2_{\tilde{\mathcal{I}}_{\tilde{\tau}}^{-1}} > -2\log (\delta\cdot c(\tilde{\mathcal{I}}_{\tilde{\tau}})),\tilde{\tau}<\infty\bigg]\\
    &= \mathbb{P}\bigg[\delta\cdot c(\tilde{\mathcal{I}}_{\tilde{\tau}})\exp\Big(\frac{1}{2}\|Q(\tilde{\tau})\|_{\tilde{\mathcal{I}}_{\tilde{\tau}}^{-1}}^2\Big) > 1\bigg]=\mathbb{P}(\delta M_{\tilde{\tau}} >1)\leq \mathbb{E}[\delta M_{\tilde{\tau}}] \leq \delta.
\end{aligned}
\end{equation*}

Thus, substituting back the definition of $c(\tilde{\mathcal{I}}_\tau)$, we obtain the uniform upper bound for all 
$\tau\geq 0$:
$$
\|Q(\tau)\|^2_{\tilde{\mathcal{I}}_\tau^{-1}} = O(\log(\det (\tilde{\mathcal{I}}_\tau)/\delta^2)).
$$
Since $\|I_t-\tilde{I}_t\|=o_p(1)$, it follows from the definition of the matrix two-norm and basic algebra that
\begin{equation}\label{eq:Q_3}
\begin{aligned} 
&\big|\|Q(\tau_t)\|^2_{{I}_t^{-1}} -\|Q(\tau_t)\|^2_{\tilde{I}_t^{-1}} \big|= \big|Q(\tau_t)^\top\{I_t^{-1}- \tilde{I}_t^{-1}\}Q(\tau_t)\big|\\& \leq \|Q(\tau_t)\|_2^2 \cdot \|I_t^{-1}- \tilde{I}_t^{-1}\|_2 = \|Q(\tau_t)\|_2^2 \cdot \|I_t^{-1}(\tilde{I}_t-I_t) \tilde{I}_t^{-1}\|_2 \\&\leq \|Q(\tau_t)\|_2^2 \cdot \|I_t^{-1}\|_2 \cdot \|\tilde{I}_t-I_t\|_2\cdot \|\tilde{I}_t^{-1}\|_2 =o(1).
\end{aligned}
\end{equation}
Therefore, 
\begin{equation}\label{eq:Q_1}
\|Q(\tau_t)\|^2_{{I}_t^{-1}} = O(\log(\det (\tilde{I}_t)/\delta^2)) + o(1).
\end{equation}

The final step in deriving the uniform upper bound for $\|Q(\tau_t)\|^2_{{I}_t^{-1}}$ is to obtain a bound on $\det (\tilde{I}_t)$.

According to the (AM-GM) inequality, $\det(\tilde{I}_t) \leq \Big(\frac{\text{tr}(\tilde{I}_t)}{d}\Big)^d$. By Assumption \ref{assump:1}.a, $\|X_i-\mu\| \leq 2L$. Since $I_t$ is a positive semi-definite matrix, we have
$$
\lambda_{\max}(\tilde{I}_t)\leq \text{tr} (\tilde{I}_t)\leq t \cdot (2L)^2.
$$
Then
\begin{equation}\label{eq:det_It}
 \det(\tilde{I}_t) \leq  \Big({t \cdot (2L)^2}/d\Big)^d = (4tL^2/d)^d.
\end{equation}

Combining the results of Equation \eqref{eq:Q_1} and \eqref{eq:det_It}, we have
\begin{equation}
\|Q(\tau_t)\|^2_{{I}_t^{-1}} = O(\log(\det (\tilde{I}_t)/\delta^2)) + o(1) = O(d\log (4tL^2/d)-2\log \delta).
\end{equation}

\textbf{Step 3.} Summary.

Based on the order derived in Equation \eqref{eq:r_1} and \eqref{eq:Q_1}, we have $\|r(\tau_t)\|_{I_t^{-1}}=o(1)$, which is negligible compared with $\|Q(\tau_t)\|_{I_t^{-1}}$. Therefore, by triangle inequality, the uniform upper bound for any $t\geq 0$ is given by 
$$
\|\hat{\beta}_t^*-\beta\|_{I_{t}} =\|U(\tau_t,\beta)\|_{I_t^{-1}} \leq  \sqrt{\|Q(\tau_t)\|_{I_t^{-1}} + \|r(\tau_t)\|_{I_t^{-1}}}= O(\sqrt{d\log (4tL^2/d)-2\log \delta}).
$$
This finishes the proof of this theorem.
\section{Proof of Theorem \ref{thm:2}}\label{appendix:thm2_proof}

\begin{theorem} (Regret Bound)
    Under Assumption \ref{assump:1}, the regret upper bound for online survival analysis under either EG, UCB, or TS algorithm for exploration, is given by
    \begin{equation*}
        \text{Regret}(T) = {O}(\sqrt{T}\log^2T).
    \end{equation*}
\end{theorem}

\textbf{Proof:}
Recall that the regret at round $T$ is defined as 
\begin{equation*}
    \text{Regret}(T) = \sum_{t=1}^T \mathbb{E}[f(S(\tau_0\mid X_t(a^*_t)))-f(S(\tau_0\mid X_t(a_t)))],
\end{equation*}
where $a^*_t$ is the optimal action, $a_t$ is the action chosen by the algorithm (e.g., EG, UCB, or TS), and the expectation is taken with respect to the randomness in $X(a)$. Given that $f$ is (uniformly) Lipschitz on $[0,1]$, the regret can be upper bounded by $C_0 \cdot \sum_{t=1}^T\mathbb{E}[\Delta_t]$ for some large constant $C_0$, which reduces the problem to establishing a regret upper bound for
\begin{equation*}
\sum_{t=1}^T\mathbb{E}[\Delta_t] = \sum_{t=1}^T [X_t(a_t)^\top\beta - X_t(a^*_t)^\top\beta].
\end{equation*}

We now claim that if $a_t$ is obtained via Cox PH model with proper exploration by either EG, UCB, or TS with Equation~\eqref{eq:EG_act}, \eqref{eq:UCB_act}, or \eqref{eq:TS_act}, then 
\begin{equation*}
    \sum_{t=1}^T \mathbb{E}[X_t(a^*)^\top\beta - X_t(a_t)^\top\beta] = \tilde{O}(\sqrt{T}).
\end{equation*}
Once this claim is established, the proof of Theorem~\ref{thm:2} is complete.

In the following subsections (Sections~\ref{appendix:bound_UCB}-\ref{appendix:bound_EG}), we provide detailed proofs for UCB, TS, and EG, respectively.

\subsection{Proof of regret bound for UCB-based survival updates}\label{appendix:bound_UCB}

For UCB-based exploration, the regret in each round $t$ can be decomposed as 
\begin{equation}\label{eq:Delta_t}
    \begin{aligned}
    \Delta_t &= X_t(a_t)^\top\beta - X_t(a^*)^\top\beta  \\&= \underbrace{\{X_t(a_t)^\top(\beta -\hat{\beta}^*_{t-1})\}}_{\text{I: exploitation error on }a_t} -\underbrace{\{X_t(a^*)^\top(\beta -\hat{\beta}^*_{t-1})\}}_{\text{II: exploitation error on }a^*} + \underbrace{\{(X_t(a_t)-X_t(a^*))^\top\hat{\beta}^*_{t-1}\}}_{\text{III: controlled by UCB}}.
    \end{aligned}
\end{equation}
For Parts I and II, this corresponds to ``regrets due to exploitation''. Bandit algorithms typically construct a \textit{confidence ellipsoid} under the matrix norm ($I_{t-1}^{-1}$-norm) to quantify the difference between $\hat{\beta}^*_{t-1}$ and $\beta$ with high probability. Specifically, we aim to find an upper bound
$$
\|\hat{\beta}_{t-1}^*-\beta\|_{I_{t-1}^{-1}}= \sqrt{(\hat{\beta}_{t-1}^*-\beta)^\top I_{t-1}^{-1}(\hat{\beta}_{t-1}^*-\beta)}\leq \alpha_t
$$
that holds uniformly over all $t\in[T]$. Typically, $\alpha_t$ serves as a hyperparameter controlling the level of exploration in UCB and depends on the confidence level $\delta$, the dimension $d$ of $X_t$, and the round $t$, as proved in Theorem \ref{thm:1}.

For Part III, this corresponds to ``regret due to exploration'', which is governed by the exploratory nature of UCB. Since UCB selects the arm with the highest upper confidence bound, we have
\begin{equation*}
    -X_t(a_t)^\top\hat{\beta}^*_{t-1} + \alpha_t \cdot \|X_t(a_t)\|_{I_{t-1}^{-1}} \geq -X_t(a^*)^\top\hat{\beta}^*_{t-1} + \alpha_t \cdot \|X_t(a^*)\|_{I_{t-1}^{-1}},
\end{equation*}
which indicates that 
\begin{equation*}
    (X_t(a_t) -X_t(a^*))^\top\hat{\beta}^*_{t-1} \leq \alpha_t \cdot \|X_t(a_t)\|_{I_{t-1}^{-1}}-\alpha_t \cdot \|X_t(a^*)\|_{I_{t-1}^{-1}}.
\end{equation*}

Going back to Equation \eqref{eq:Delta_t}, we further have

\begin{equation}
    \begin{aligned}
    \Delta_t &= X_t(a^*)^\top\beta - X_t(a_t)^\top\beta \\&= \underbrace{\{X_t(a^*)^\top(\beta -\hat{\beta}^*_{t-1})\}}_{\text{I: exploitation error on }a^*} -\underbrace{\{X_t(a_t)^\top(\beta -\hat{\beta}^*_{t-1})\}}_{\text{II: exploitation error on }a_t} + \underbrace{\{(X_t(a^*)-X_t(a_t))^\top\hat{\beta}^*_{t-1}\}}_{\text{III: controlled by UCB}}\\
    &\leq \|\hat{\beta}_{t-1}-\beta\|_{I_{t-1}} \|X_t(a^*)\|_{I_{t-1}^{-1}} + \|\hat{\beta}_{t-1}-\beta\|_{I_{t-1}} \|X_t(a_t)\|_{I_{t-1}^{-1}} + \alpha_t \|X_t(a_t)\|_{I_{t-1}^{-1}} - \alpha_t \|X_t(a^*)\|_{I_{t-1}^{-1}}\\&\leq \alpha_t \|X_t(a^*)\|_{I_{t-1}^{-1}}+ \alpha_t \|X_t(a_t)\|_{V_{t-1}^{-1}} +  \alpha_t \|X_t(a_t)\|_{I_{t-1}^{-1}} - \alpha_t \|X_t(a^*)\|_{I_{t-1}^{-1}}\\& = 2 \alpha_t \|X_t(a_t)\|_{I_{t-1}^{-1}}.
    \end{aligned}
\end{equation}

Therefore, to bound $\sum_{t=1}^T \Delta_t$, we divide the remainder of the proof into three steps. In Step 1, we determine the order of $\|\hat{\beta}_{t-1}-\beta\|_{V_{t-1}^{-1}}$ (or $\alpha_t$); in Step 2, we detail the order of $\|X_t(a_t)\|_{V_{t-1}^{-1}}$. Finally, in Step 3, we summarize the results and complete the proof.

\textbf{Step 1.} Derive the order of $\alpha_t$ such that $\|\hat{\beta}_{t-1}^*-\beta\|_{I_{t-1}}\leq \alpha_t$.

From the result of Theorem~\ref{thm:1}, it follows directly that for any $\delta>0$, the following holds uniformly for all rounds $t\geq 0$:
$$
\|\hat{\beta}_t^*-\beta\|_{I_{t}} \leq  O(\sqrt{d\log (4tL^2/d)-2\log \delta}).
$$
Therefore, we can set $\alpha_t= O(\sqrt{d\log (4tL^2/d)-2\log \delta}) $ as the UCB exploration hyperparameter. Since $\alpha_t$ grows roughly like $\sqrt{\log t}$ with $t$, it increases slowly over time and is consistent with the order commonly used in the literature on generalized linear bandits \citep{filippi2010parametric}.

\textbf{Step 2.} Derive the order of $\|X_t(a_t)\|_{I_{t-1}^{-1}}$.

Recall that 
$$
\begin{aligned}
\tilde{I}_{t} &= \sum_{i=1}^N\int_0^{\tau_t} \left[\{X_i-\mu(s)\}\{X_i-\mu(s)\}^\top\right]A_i(\tau_t,ds).
\end{aligned}
$$
Using similar logic as in Equation \eqref{eq:Q_3}, one can derive that $\|X_t(a_t)\|_{I_{t-1}^{-1}}  = \|X_t(a_t)\|_{\tilde{I}_{t-1}^{-1}} + o(1)$.
Therefore, we only need to bound $\|X_t(a_t)\|_{\tilde{I}_{t-1}^{-1}}$.

According to Assumption \ref{assump:1}.a and \ref{assump:1}.c, 
$$
\|X_t(a_t)\|^2_{\tilde{I}_{t-1}^{-1}} \leq \|X_t(a_t)\|_{2}^2 \cdot \|\tilde{I}_{t-1}^{-1}\|_{2} \leq L^2/\lambda_{\min}(\tilde{I}_{t-1})\leq dL^2/\kappa (t-1).
$$
Thus,
$$
\sum_{t=1}^T \|X_t(a_t)\|^2_{\tilde{I}_{t-1}^{-1}} \leq \sum_{t=1}^T dL^2/\kappa t = O\Big(\frac{dL^2}{\kappa} \log T\Big),
$$
which gives us $\sum_{t=1}^T \|X_t(a_t)\|^2_{{I}_{t-1}^{-1}} = O\Big(\frac{dL^2}{\kappa} \log T\Big)$.

\textbf{Step 3.} Summary.

Returning to Equation~\eqref{eq:Delta_t} and combining it with the results from Steps 1 and 2, we further obtain
$$
\begin{aligned}\sum_{t=1}^T \Delta_t& \leq 2\sum_{t=1}^T\alpha_t \|X_t(a_t)\|_{I_{t-1}^{-1}}\leq 2\sqrt{T \sum_{t=1}^T \alpha_t^2\|X_t(a_t)\|^2_{I_{t-1}^{-1}}}\\
&\leq 2\sqrt{ \alpha_T^2T\sum_{t=1}^T \|X_t(a_t)\|^2_{I_{t-1}^{-1}}}= 2 \alpha_T\sqrt{T\sum_{t=1}^T \|X_t(a_t)\|^2_{I_{t-1}^{-1}}} \\
&\lesssim O\Big(\sqrt{d\log (4L^2T/d) -2\log \delta}\Big)\cdot O\Big(\frac{\sqrt{d}L}{\sqrt{\kappa}} \sqrt{T} \sqrt{\log T}\Big)\\&=O(d\sqrt{T}\log T)\end{aligned}
$$
Therefore, the regret upper bound of UCB-based exploration for online survival analysis is of order $O(d\sqrt{T}\log T)$.

\subsection{Proof of regret bound for TS-based survival updates}\label{appendix:bound_TS}

Next, we derive the regret bound for TS-based exploration.
We similarly decompose the regret by 
\begin{equation}\label{eq:Delta_TS}
\begin{aligned}\Delta_t &= X_t(a_t)^\top\beta - X_t(a^*)^\top\beta \\&= \underbrace{\{X_t(a_t)^\top(\beta -\tilde{\beta}_{t-1})\}}_{\text{I: exploitation error on }a_t} -\underbrace{\{X_t(a^*)^\top(\beta -\tilde{\beta}_{t-1})\}}_{\text{II: exploitation error on }a^*} + \underbrace{\{(X_t(a_t)-X_t(a^*))^\top\tilde{\beta}_{t-1}\}}_{\text{III: }\leq 0, \text{ controlled by TS}}\\& \leq (X_t(a_t)-X_t(a^*))^\top(\tilde{\beta}_{t-1}-\beta).
\end{aligned}
\end{equation}
Since $\tilde{\beta}_{t}$ is chosen from the posterior, we can further decompose 
$$
\tilde{\beta}_{t}-\beta = \underbrace{(\tilde{\beta}_{t}-\hat{\beta}_t^* )}_{\text{exploration error}} + \underbrace{(\hat{\beta}_{t}^*-\beta)}_{\text{exploitation error}}.
$$

\textbf{Step 1.} Derive the uniform upper bound for $\|\tilde{\beta}_{t}-{\beta} \|_{I_t}$.

By Bernstein-von Mises Theorem, $(\tilde{\beta}_t-\hat{\beta}^*_t)\rightarrow\mathcal{N}(0,I_t^{-1}(\beta))$,
where $\beta$ is the true value. Therefore, according to Slusky theorem,
$$
\tilde{\beta}_t = \hat{\beta}_t^* + I_t^{-1/2} Z_t + R_t,
$$
where $Z_t \sim\mathcal{N}(\boldsymbol{0}_d,I_d)$, and $R_t$ is the remainder term. Thus,
$$
\|\tilde{\beta}_t-\hat{\beta}_t^*\|_{I_{t}} \leq \|I_t^{-1/2} Z_t\|_{I_{t}}  + \|R_t\|_{I_{t}}   = \|Z_t\|_2 + \|R_t\|_{I_{t}}.
$$
Therefore, bounding $\|\tilde{\beta}_t-\hat{\beta}_t^*\|_{I_{t}}$ is equivalent to bound $\|Z_t\|_2$ and $\|R_t\|_{I_{t}}$.

First, let’s quantify $\|Z_t\|_2$.

By the Laurent-Massart inequality \citep{laurent2000adaptive},
$$
\mathbb{P}(\|Z_t\|_2^2> d + 2\sqrt{dx} + 2x)\leq e^{-x}.
$$
Take $x = \log (T/\delta)$, gives per-step probability bound $\delta/T$. Thus, with probability at least $1-\delta/T$, we have $\|Z_t\|_2^2\leq d + 2\sqrt{d\log (T/\delta)} + 2\log (T/\delta):= b_T$.

For any $t\in[T]$, 
$$
\mathbb{P}(\exists t\leq T,\text{s.t. } \|Z_t\|_2^2> b_T) \leq \sum_{t=1}^T \mathbb{P}(\|Z_t\|_2^2> b_T)\leq T\cdot \delta/T = \delta.
$$
Thus, with probability $1-\delta$, we have  $\|Z_t\|_2^2 \leq d + 2\sqrt{d\log (T/\delta)} + 2\log (T/\delta) =O((\sqrt{d}+1)\log(T/\delta))$ for all $t\leq T$.

Next, we derive the uniform upper bound for $\|R_t\|_{I_{t}}$.

This part requires the approximation accuracy of Bayesian update to the MLE estimator for Gaussian prior. 
As the leading term $I_t^{-1/2}Z_t = O_p(1/\sqrt{t})$ dominates,  $R_t$ is a remainder term with $R_t= O_p(1/t)$ and $\|R_t\|_{I_t} =O_p(1/\sqrt{t})$, which is a point-wise bound for each $t$. There exists constants $L_1$ and $L_2$, such that
\begin{equation*}
    \|R_t\|_{I_{t}} \leq \begin{cases}
L_1, & \text{if } t \leq t_0, \\
L_2\cdot t^{-1/2}, & \text{otherwise.}
\end{cases}
\end{equation*}
The uniform bound is then 
\begin{equation*}
    \sup_{t\geq 1} \|R_t\|_{I_{t}} \leq \max \{L_1,L_2\cdot t_0^{-1/2}\} = O(1). 
\end{equation*}
Thus, 
$$
\|\tilde{\beta}_t-\hat{\beta}^*_t\|_{I_{t}} \leq   \|Z_t\|_2 + \|R_t\|_{I_{t}}\leq O((\sqrt{d}+1)\log(t/\delta)) + O(1) \\=O((\sqrt{d}+1)\log(t/\delta)).
$$
According to Theorem \ref{thm:1}, $
\|\hat{\beta}_t^*-\beta\|_{I_{t}} = O(\sqrt{d\log (4tL^2)-\log \delta}) = O(\sqrt{d\log t})$. Therefore, 
$$
\|\tilde{\beta}_{t}-\beta\|_{I_{t}} \leq \underbrace{\|\tilde{\beta}_{t}-\hat{\beta}_t^* \|_{I_{t}}}_{\text{exploration error}} + \underbrace{\|\hat{\beta}_{t}^*-\beta\|_{I_{t}}}_{\text{exploitation error}}\lesssim O(\sqrt{d}\log t).
$$

\textbf{Step 2.} Summary.

Now let's decompose $\Delta_t$ for TS-based survival updates and derive the final regret bound.

Following similar procedure as Step 2 of Section \ref{appendix:bound_UCB}, we have $\sum_{t=1}^T \|X_t(a_t)\|^2_{{I}_{t-1}^{-1}} = O\big(\frac{dL^2}{\kappa} \log T\big)$, and $\sum_{t=1}^T \|X_t(a^*)\|^2_{{I}_{t-1}^{-1}} = O\big(\frac{dL^2}{\kappa} \log T\big)$. Based on Equation \eqref{eq:Delta_TS}, we have
$$
\begin{aligned}\sum_{t=1}^T \Delta_t& \leq 2\|\tilde{\beta}_{t-1}-\beta\|_{I_{t-1}}\sqrt{T\sum_{t=1}^T \|X_t(a_t)\|^2_{I_{t-1}^{-1}}} \lesssim O(\sqrt{d}\log T)\cdot O\big(\frac{\sqrt{d}L}{\sqrt{\kappa}} \sqrt{T\log T}\big)\\&=O(d\sqrt{T}\log T \sqrt{\log T}).
\end{aligned}
$$
Therefore, the regret upper bound of TS-based exploration for online survival analysis is of order $O(d\sqrt{T}\log T\sqrt{\log T})$.

\subsection{Proof of regret bound for EG-based survival updates}\label{appendix:bound_EG}
Finally, we derive the regret upper bound for EG-based survival updates. Given the regret analyses for UCB and TS, the EG-based exploration follows a similar proof structure and can be established under an appropriately chosen exploration probability $\epsilon_t$.
\begin{equation*}
    \Delta_t = X_t(a_t)^\top\beta - X_t(a^*)^\top\beta   = \epsilon_t\mathbb{E}[\Delta_t\mid \text{explore}] + (1-\epsilon_t)\mathbb{E}[\Delta_t\mid \text{exploit}].
\end{equation*}
For the regret due to exploitation (i.e., the second term), 
$$
\mathbb{E}[\Delta_t\mid \text{exploit}] \leq  X_t(a_t)^\top(\beta-\hat{\beta}_{t-1}^*) + X_t(a^*)^\top(\hat{\beta}_{t-1}^*-\beta).
$$
By Theorem \ref{thm:1}, 
$$
|X(a)^\top (\hat{\beta}_{t-1}^*-\beta)|\leq \|\hat{\beta}_{t-1}^*-\beta\|_{I_{t-1}}\cdot\|X(a)\|_{I_{t-1}^{-1}}\leq  \alpha_{t-1}\cdot\|X(a)\|_{I_{t-1}^{-1}},
$$
where $\alpha_t = O(\sqrt{d\log (4tL^2/d)-2\log \delta})$.

For the regret due to exploration (i.e., the first term), by Assumption~\ref{assump:1}.b, we have $\mathbb{E}[\Delta_t\mid \text{explore}] \leq \epsilon_t\Delta_{\max}$, and since $\Delta_{\max}$ is bounded, this term is of the same order as $\epsilon_t$. Therefore,
\begin{equation*}
\begin{aligned}
    \sum_{t=1}^T\Delta_t 
    &= \sum_{t=1}^T\epsilon_t\mathbb{E}[\Delta_t\mid \text{explore}] + \sum_{t=1}^T(1-\epsilon_t)\mathbb{E}[\Delta_t\mid \text{exploit}] \\
    &\leq \Delta_{\max}\sum_{t=1}^T \epsilon_t +   2\alpha_T\cdot\sum_{t=1}^T\|X(a)\|_{I_{t-1}^{-1}}\\
    &\leq \Delta_{\max}\sum_{t=1}^T \epsilon_t + 2\alpha_T\cdot\sqrt{T\sum_{t=1}^T\|X(a)\|_{I_{t-1}^{-1}}^2}\\
    &\lesssim \Delta_{\max}\sum_{t=1}^T \epsilon_t +  O(\sqrt{d\log (4TL^2/d)-2\log \delta})\cdot O\Big(\frac{\sqrt{d}L}{\sqrt{\kappa}} \sqrt{T\log T}\Big)\\
    & \lesssim  \Delta_{\max}\sum_{t=1}^T \epsilon_t  + O(d\sqrt{T}\log T).
\end{aligned}
\end{equation*}
If we choose $\epsilon_t = \min\{1,c/t\}$ for some constant $c$, then $\sum_{t=1}^T \epsilon_t = O(\log T)$, making the first term above negligible compared with exploitation term $O(d\sqrt{T}\log T)$. Therefore, the regret upper bound of EG-based exploration for online survival analysis is of order $O(d\sqrt{T}\log T)$.

This finishes the proof of Theorem \ref{thm:2}.

\end{document}